\documentclass[journal]{IEEEtran}

\usepackage{multirow}%
\usepackage{booktabs}%
\usepackage{makecell}
\usepackage{pifont}
\usepackage[table]{xcolor}  
\definecolor{lightgray}{gray}{0.9}
\usepackage{algorithmic}
\usepackage{algorithm}
\usepackage{url}
\usepackage{hyperref}
\usepackage{enumitem}
\usepackage{graphicx}
\usepackage{newtxtext}
\usepackage[caption=false,font=footnotesize]{subfig}

\usepackage{amsmath,amsfonts,bm}
\usepackage{amssymb}
\usepackage{mathtools}
\usepackage{amsthm}

\newcommand{\T}{{\hspace{-0.25ex}\top\hspace{-0.25ex}}}

\newcommand{\dif}{\mathrm{d}}

\newcommand{\bE}{\mathbb{E}}
\newcommand{\bR}{\mathbb{R}}

\newcommand{\cX}{\mathcal{X}}
\newcommand{\cY}{\mathcal{Y}}

\newcommand{\cH}{\mathcal{H}}
\newcommand{\cS}{\mathcal{S}}
\newcommand{\calF}{{\mathcal F}}
\newcommand{\calP}{{\mathcal P}}
\newcommand{\calW}{{\mathcal W}}
\newcommand{\calZ}{{\mathcal Z}}
\newcommand{\cL}{\mathcal{L}}
\newcommand{\cD}{\mathcal{D}}

\newcommand{\bx}{\boldsymbol{x}}
\newcommand{\btheta}{\theta}
\newcommand{\bz}{\boldsymbol{z}}
\newcommand{\bw}{\boldsymbol{w}}
\newcommand{\bK}{\boldsymbol{K}}
\newcommand{\bk}{\boldsymbol{k}}

\newcommand{\ptr}{p_\mathrm{tr}}
\newcommand{\pte}{p_\mathrm{te}}
\newcommand{\tr}{\mathrm{tr}}
\newcommand{\te}{\mathrm{te}}
\newcommand{\val}{\mathrm{v}}
\newcommand{\mf}{\boldsymbol{f}_\theta}

\newcommand{\tnabla}{\tilde{\nabla}_{\boldsymbol{\theta}}}
\newcommand{\tgrad}{\tilde{\nabla}_{s}}
\newcommand{\grad}{\nabla_{s}}
\newcommand{\bbeta}{{\boldsymbol{\beta}}}
\newcommand{\mathbbR}{\mathbb{R}}

\newcommand{\Str}{\cS^\tr}
\newcommand{\Sval}{\cS^\val}

\newcommand{\Ltr}{\cL^\tr}
\newcommand{\Lval}{\cL^\val}

\newcommand{\nparams}{b}

\theoremstyle{plain}
\newtheorem{theorem}{Theorem}[section]

\theoremstyle{definition}

\newtheorem{assumption}[theorem]{Assumption}
\theoremstyle{remark}

\usepackage[numbers,sort&compress]{natbib}

\begin{document}


\title{Accelerated Dynamic Importance Weighting with Versatile Divergence-Minimizing Estimators}





\author{
Tongtong Fang$^1$,
Nan Lu$^{2,3}$,
Gang Niu$^3$,
Kenji Fukumizu$^1$,
Masashi Sugiyama$^{3,4}$
\thanks{
$^1$ The Institute of Statistical Mathematics, Tokyo, Japan.
}
\thanks{
$^2$ University of Bristol, Bristol, U.K.
}
\thanks{
$^3$ RIKEN Center for Advanced Intelligence Project, Tokyo, Japan.
}
\thanks{
$^4$ The University of Tokyo, Tokyo, Japan.
}
}


\markboth{Journal of \LaTeX\ Class Files,~Vol.~14, No.~8, August~2021}%
{Shell \MakeLowercase{\textit{et al.}}: A Sample Article Using IEEEtran.cls for IEEE Journals}


\maketitle

\begin{abstract}
\emph{Importance weighting} (IW) is a golden solver for \emph{joint distribution shift}, where the joint distributions differ between the training and test data.
To solve this problem, IW estimates test-to-training density ratios as importance weights and reweights the training losses accordingly.
Recent advances in \emph{dynamic IW} (DIW) integrate weight estimation into model training, enabling scalable IW for deep models and achieving strong performance on large modern datasets. 
Despite its promise, DIW remains limited in two aspects.
First, it incurs substantial computational overhead by solving a \emph{kernel mean matching}~(KMM)-induced optimization problem to convergence in every mini-batch.
Second, it relies solely on KMM for weight estimation, whereas the IW literature contains diverse estimation methods based on different divergence measures.
In this paper, we propose \emph{accelerated DIW}~(ADIW), a unified and efficient IW framework for deep learning under joint distribution shift.
ADIW performs a few lightweight \emph{projected gradient descent} updates that warm-start from previously updated weights, substantially improving efficiency.
Moreover, ADIW generalizes DIW into a unified divergence-minimization framework that supports diverse weight-estimation methods in a plug-and-play manner, including those based on the \emph{Kullback--Leibler divergence}, \emph{squared distance}, and \emph{Wasserstein-1 distance}.
We establish convergence guarantees for ADIW under mild conditions, and empirical results demonstrate that ADIW achieves state-of-the-art 
IW performance while being substantially more efficient.

\end{abstract}


\section{Introduction}
\label{sec:intro}
\IEEEPARstart{D}{\emph{istribution}} \emph{shift} \citep{quionero2009dataset, pan2009survey}
poses a fundamental challenge for deep learning, as models trained on one distribution often fail to generalize to test data drawn from a different distribution.
{
The most general form of distribution shift is \emph{joint distribution shift}, where the joint densities of the training and test data are different without additional assumptions on the shift structure.}
A classic approach to {joint} distribution shift is \emph{importance weighting}~(IW) \citep{shimodaira2000improving,sugiyama2007covariate}, 
{which, in the classification setting considered in this paper, consists of two steps:}
(i) \emph{weight estimation}~{(WE)}, where test-to-training density ratios are estimated using a tiny validation set from the test distribution;
and (ii) \emph{weighted classification}~{(WC)}, where training losses are reweighted accordingly.

\begin{figure*}[t]
    \centering
    \subfloat[Illustrations of DIW (left) and ADIW (right).]{
        \includegraphics[width=0.71\linewidth]{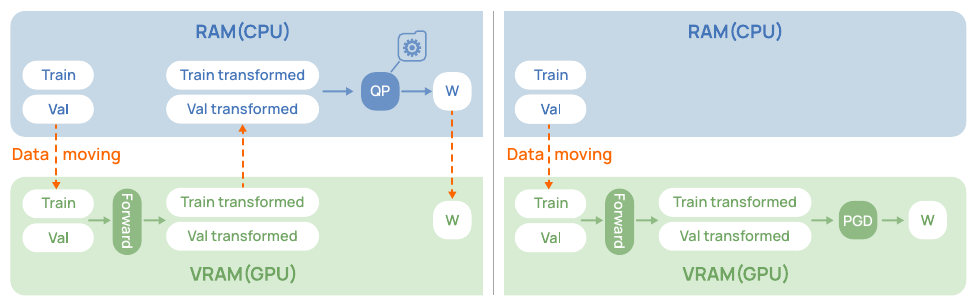}
        \label{fig:pipeline}
        }
        \hfill
    \subfloat[Training time.]{
        \centering
        \includegraphics[width=0.255\linewidth]{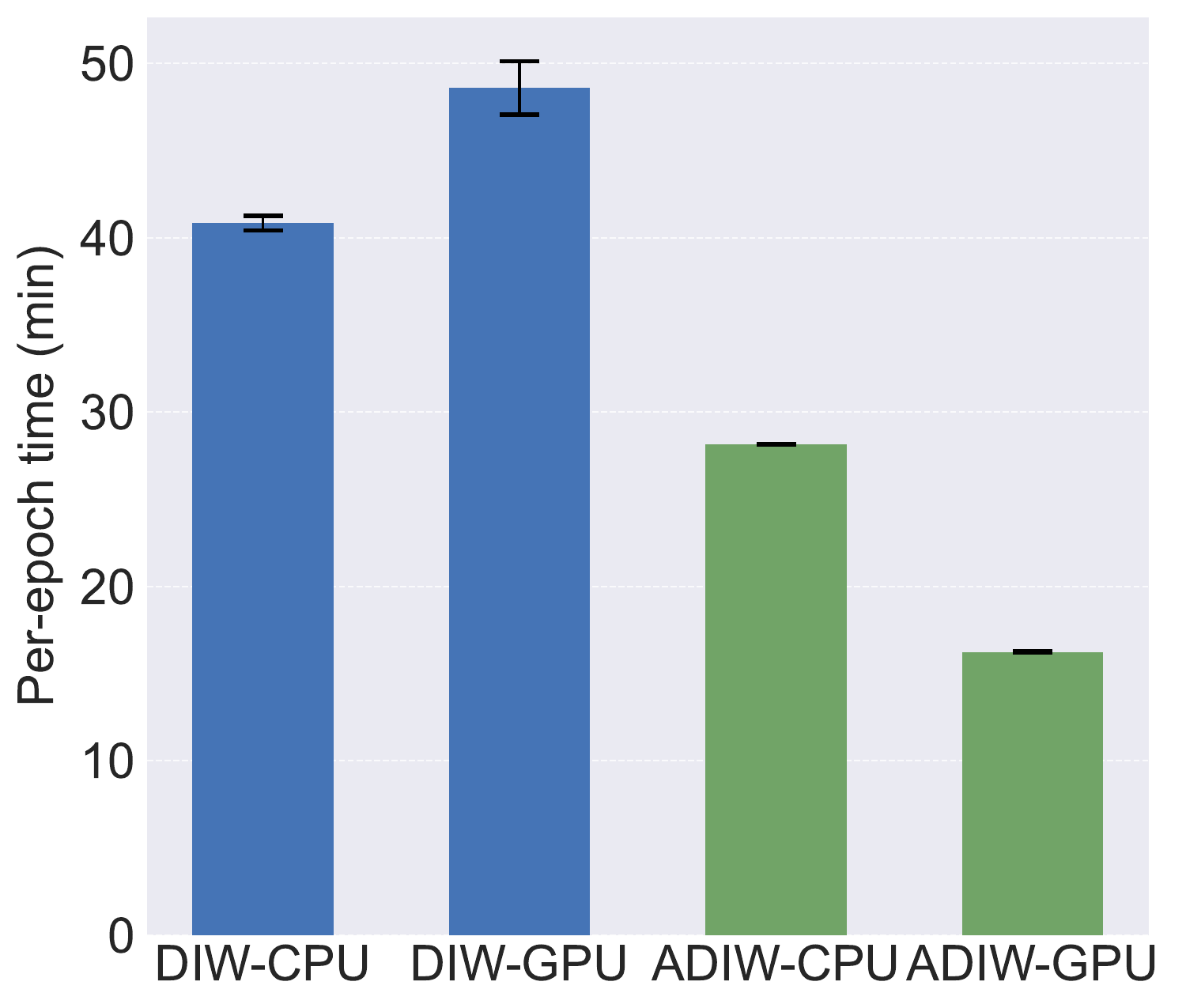}
        \label{fig:time}}

    \vspace{4pt}

    \footnotesize
    \raggedright
    (a) In DIW, weight estimation (WE) is implemented as a quadratic program (QP) solved until convergence in each mini-batch on the CPU using an external solver, incurring overhead from CPU-GPU data transfers and over-solving WE at each iteration.
    In contrast, ADIW optimizes the weights via a few \emph{projected gradient descent} (PGD) updates, enabling efficient, fully GPU-accelerated end-to-end training without external QP solvers.
    (b) Per-epoch end-to-end training time on ImageNet-1K with 0.4 symmetric label noise (minutes). Error bars show standard deviation over 3 runs. DIW-CPU solves WE with CVXOPT on the CPU; DIW-GPU uses cuOSQP on GPU but involves CPU-side preprocessing and additional CPU-GPU data transfers; ADIW-CPU and ADIW-GPU perform PGD-based WE on CPU and fully on GPU, respectively. All methods employ kernel mean matching for WE.

    \vspace{1pt}
    \caption{Comparison of DIW and the proposed ADIW.}
    \label{fig:overview}
\end{figure*}
While effective for low-dimensional data and linear models, 
IW {often} degrades on high-dimensional and complex data and deep models \citep{byrd2018effect}.
To address this limitation, \emph{dynamic IW}~(DIW) \citep{fang2020rethinking} 
introduced a \emph{feature extractor} from the deep {model}
and performed WE on {feature representations}
{in every mini-batch}. 
Despite its effectiveness on modern datasets, DIW remains limited in computational efficiency and {the coverage of WE methods,}
as detailed below.

\textbf{Computational Limitations.}
{
DIW has two major computational bottlenecks.
(i) \emph{CPU-GPU data transfer overhead.} 
{DIW relies on quadratic programming (QP) solvers for \emph{kernel mean matching}~(KMM) \citep{huang2007correcting} in WE}, which are typically CPU-centric and thus incur frequent CPU-GPU data transfers (see Figure~\ref{fig:overview}(\subref{fig:pipeline}); CPU and GPU denote the central and graphics processing units, respectively).
Most mainstream QP solvers (e.g., CVX, CVXOPT, CPLEX, and Gurobi)~\footnote{See \url{https://cvxr.com/cvx/} for CVX, \url{https://cvxopt.org} for CVXOPT, \url{https://www.ibm.com/products/ilog-cplex-optimization-studio} for CPLEX, and \url{https://www.gurobi.com} for Gurobi.} offer limited or no GPU support, and even GPU-enabled solvers such as cuOSQP \citep{cuosqp} typically require CPU-side inputs for problem setup and thus still have CPU-GPU data transfers (see Figure~\ref{fig:overview}(\subref{fig:time})). 
Moreover, even if QP-based WE could be performed entirely on GPU, not all WE objectives are QP problems (e.g., \emph{Kullback--Leibler (KL) importance estimation procedure}~(KLIEP) \citep{sugiyama2008direct}), making a unified implementation difficult. 
(ii) \emph{Over-solving WE.}
In DIW, a KMM problem is solved until convergence in each mini-batch,
which is unnecessary since the representations used for WE change consecutively with model update.}

\textbf{Coverage Limitations.}
Classic two-step IW methods are each built upon a specific \emph{divergence} {in the WE step}, such as \emph{maximum mean discrepancy}~(MMD) in KMM, \emph{KL divergence} in KLIEP, and \emph{squared distance} in \emph{least-squares importance fitting} (LSIF)~\citep{kanamori2009least}.
{The choice of divergence determines how distribution shift is measured and affects both WE and WC through the estimated weights.
This motivates a unified framework that supports diverse WE methods based on different divergences within a single implementation, enabling systematic comparison and selection.}
More broadly, Sugiyama et al.~\cite{sugiyama2012density1} unified many IW methods via \emph{Bregman divergence} \citep{bregman1967relaxation} minimization, but
it is restricted to the Bregman family, does not provide a unified implementation pipeline for different methods, and does not readily support end-to-end deep learning.
Although DIW is well suited to deep learning, it remains tightly coupled to a single {WE} method, i.e., KMM.
Recent extensions to alternative {WE} methods are either application-specific (i.e., language model alignment \citep{lodkaewimportance}) or setting-specific (i.e., positive-unlabeled learning \citep{kumagai2025importance}), while still inheriting the computational inefficiency of DIW.

In summary, there is still no principled and general {IW} framework that can flexibly integrate diverse {WE} methods into deep learning pipelines in a computationally efficient manner.


To address these challenges, we propose \emph{accelerated DIW}~(ADIW), 
a unified and efficient IW framework for deep learning (see Figure~\ref{fig:overview}(\subref{fig:pipeline})).
ADIW formulates WE as density matching between transformed training and test distributions, where the transformation is induced by a feature extractor from the deep model, and density matching is done by minimizing a specified divergence measure.
Instead of solving WE until convergence on CPUs, ADIW performs WE using a small number of lightweight \emph{projected gradient descent} (PGD) \citep{polyak1966constrained} updates (typically a single step) directly on GPUs, {with each update warm-started from the previously updated weights.}
This design removes the dependency on external QP solvers and improves efficiency by avoiding unnecessary CPU-GPU data transfers.
Moreover, ADIW supports versatile {WE} methods in a plug-and-play manner, enabling seamless switching and fair comparison within a unified pipeline, which is difficult to achieve with existing heterogeneous implementations.

Overall, ADIW offers a unified and computationally efficient IW framework for deep learning {under joint distribution shift}.
Our contributions can be summarized as follows:
\begin{itemize}
    \item \textbf{Coverage.} 
    {
    We propose ADIW, a unified IW framework that supports a variety of divergences used in classic IW (including the aforementioned MMD, KL divergence, and squared distance) within an end-to-end deep learning pipeline in a plug-and-play manner.}
    \item \textbf{Algorithmic.}  
    We analyze the computational bottleneck of DIW, develop an \emph{efficient} PGD-based implementation of ADIW,
    and prove convergence of the ADIW algorithm to a stationary point of the WC objective. 
    \item \textbf{Practical.} We instantiate several classic IW methods within ADIW, introduce a \emph{novel} Wasserstein-1 distance variant, and provide systematic empirical comparisons.
    \item \textbf{Empirical.} We demonstrate that ADIW achieves comparable or better performance than DIW while substantially reducing training time (see Figure~\ref{fig:overview}(\subref{fig:time})).
\end{itemize}

{
The remainder of this paper is organized as follows.
Section~\ref{sec:background} introduces the problem setting and reviews IW and DIW.
Section~\ref{sec:ADIW} presents ADIW, followed by instantiations of IW methods in Section~\ref{sec:examples}.
Experimental results are reported in Section~\ref{sec:ADIW_exp}.
Additional details on the convergence analysis, discussions, and experiments are provided in the appendices.}

\section{Background}
\label{sec:background}
This section provides the necessary background for developing the proposed ADIW framework. 
First, we introduce the problem setting and then review classic IW and its deep learning extension, DIW.

\subsection{Problem Setting}
Let $\bx$ and $y$ denote the input and output random variables, respectively.
In this work, we consider the problem of \emph{joint distribution shift}, where the joint densities of the training and test data are different, i.e., $\ptr(\bx,y)\neq\pte(\bx,y)$.

Let $\cX\subset\mathbbR^d$ be the input domain and $\cY=\{1,\ldots,C\}$ be the label set for a $C$-class classification problem, where $d$ denotes the input dimension.
The objective is to minimize the classification risk:
\begin{align}\textstyle
\label{eq:risk}%
R(\mf) = \bE_{p_\te}[\ell(\mf(\bx),y)],
\end{align}
where $\mf:\cX\to\bR^C$ is a classifier to be trained parameterized by $\theta$, $\ell:\bR^C\times\cY\to(0,+\infty)$ is a surrogate loss function for training $\mf$, such as the \emph{softmax cross-entropy loss}, and $\bE_{p_\te}$ is the expectation over ${\pte(\bx,y)}$.

In this setting, we are given a set of training data $\cD_\tr=\{(\bx_i^\tr,y_i^\tr)\}_{i=1}^{n_\tr}\overset{\text{i.i.d.}}\sim\ptr(\bx,y)$
and a small validation set 
$\cD_\val=\{(\bx_i^\val,y_i^\val)\}_{i=1}^{n_\val}\overset{\text{i.i.d.}}\sim\pte(\bx,y)$, where $n_\tr$ and $n_\val$ are the sample sizes for $\cD_\tr$ and $\cD_\val$, and typically $n_\tr \gg n_\val$.
The goal is to reliably estimate the risk using $\cD_\tr$ and $\cD_\val$, and to train $\mf$ by minimizing a certain empirical risk on $\cD_\tr$ that {performs better}
than training $\mf$ solely on $\cD_\val$.

\begin{table*}
\centering
\caption{
Comparison of classic Importance Weighting (IW), Dynamic IW (DIW), and the proposed Accelerated DIW (ADIW) method.
}
\label{tab:divergence_comparison}
\begin{tabular}{lllcccccc}
\toprule
\textbf{Category} & \textbf{Divergence} & \textbf{Method} &
\textbf{\makecell[c]{Weight \\ model}} &
\textbf{\makecell[c]{Deep learning \\ compatibility}} & 
\textbf{\makecell[c]{WE \\ optimization}} &
\textbf{\makecell[c]{Overall \\ optimization}} &
\textbf{\makecell[c]{Comp. \\ cost}} & \textbf{\makecell[c]{Practical \\ adapt.}} \\
\midrule
\multirow{4}{*}{IPM} 
 & \multirow{3}{*}{MMD} 
    & KMM \citep{huang2007correcting} & \makecell[c]{None} & \ding{55} & \makecell[c]{Convex QP \\ (all data, once)} &  \makecell[c]{Two-step \\ (WE $\to$ WC)} & -- & Low \\
 &  & DIW \citep{fang2020rethinking} & \makecell[c]{None} & \ding{51} & \makecell[c]{Convex QP \\ (per batch)} & \makecell[c]{End-to-end \\ (WE $\circlearrowright$ WC)} & High & High \\
 &  & ADIW-KM (ours) & \makecell[c]{None} & \ding{51} & \makecell[c]{PGD \\ (per batch)} & \makecell[c]{End-to-end \\ (WE $\circlearrowright$ WC)} & Low & High \\
 & Wasserstein-1 & ADIW-WA (ours) & \makecell[c]{None} & \ding{51} & \makecell[c]{PGD$^\ast$ \\ (per batch)} & \makecell[c]{End-to-end \\ (WE $\circlearrowright$ WC)} & Low & High \\
\midrule
\multirow{2}{*}{\makecell[c]{$f$-divergence \\ \& Bregman}}
 & \multirow{2}{*}{KL} 
    & KLIEP \citep{sugiyama2008direct} & \makecell[c]{Parametric } & \ding{55} & \makecell[c]{Convex\\ (all data, once)} & \makecell[c]{Two-step \\ (WE $\to$ WC)} & -- & Low \\
 &  & ADIW-KL (ours) & \makecell[c]{Parametric } & \ding{51} & \makecell[c]{PGD \\ (per batch)} & \makecell[c]{End-to-end \\ (WE $\circlearrowright$ WC)} & Low & High \\
\midrule
\multirow{2}{*}{Bregman} 
 & \multirow{2}{*}{Squared Loss} 
    & LSIF \citep{kanamori2009least} & \makecell[c]{Parametric} & \ding{55} & \makecell[c]{Convex QP \\ (all data, once)} & \makecell[c]{Two-step \\ (WE $\to$ WC)} & -- & Low \\
 &  & ADIW-LS (ours) & \makecell[c]{Parametric} & \ding{51} & \makecell[c]{PGD \\ (per batch)} & \makecell[c]{End-to-end \\ (WE $\circlearrowright$ WC)} & Low & High \\
\bottomrule
\end{tabular}

\vspace{2pt}
\parbox{\linewidth}{\footnotesize
\textit{Abbreviations:}
WE = weight estimation;
WC = weighted classification;
Comp. cost = per-iteration computational cost (comparable only across end-to-end methods);
Practical adapt. = ability to handle complex datasets;
QP = quadratic programming;
PGD = projected gradient descent;
LS = least squares.

\textit{Optimization Schemes:}
Two-step methods optimize WE first and then WC, denoted as ``WE $\to$ WC'',
while end-to-end methods jointly optimize WE and WC, denoted as ``WE $\circlearrowright$ WC''.
$^\ast$ For ADIW-WA, PGD requires solving an auxiliary maximization problem to obtain the optimal critic for gradient computation.

\textit{Notes:}
This comparison is conceptual only.
The improved adaptability of DIW and ADIW arises from their compatibility with deep models in the WE step.
Empirical comparisons with classic IW methods are reported in Fang et al.~\cite{fang2020rethinking}, while comparisons between DIW and ADIW are provided in Section~\ref{sec:ADIW_exp}.

}
\end{table*}

\subsection{Importance Weighting}
\label{sec:iw}
\emph{Importance weighting}~(IW)
is a widely used technique for addressing distribution shift by reweighting training samples to match the test distribution.
{
Assuming $p_\te(\bx, y)$ is \emph{absolutely continuous} w.r.t. $p_\tr(\bx, y)$, i.e., $p_\te(\bx, y) > 0$ implies $p_\tr(\bx, y) > 0$,
the risk in Eq.~\eqref{eq:risk} can be rewritten as}
an importance-weighted expectation of $\ell$ over $p_\tr(\bx, y)$:
\begin{align}
    \bE_{p_\te}\left[\ell(\mf(\bx), y)\right] = \bE_{p_\tr}\left[w^*(\bx, y)\ell(\mf(\bx), y)\right],
\end{align}
where $w^*(\bx, y) = \frac{p_\te(\bx, y)}{p_\tr(\bx, y)}$ is defined as the \emph{weight} or \emph{importance}, compensating for the difference between the training and test distributions.

Classic IW is usually formulated as a two-step approach:

(i) \emph{Weight Estimation} (WE).
In the WE step, $w^*(\bx, y)$ is estimated from $\cD_\tr$ and $\cD_\val$.

(ii) \emph{Weighted Classification} (WC). 
In the WC step, we plug-in the estimated weights to train $\mf$.

{
Representative classic IW methods include \emph{kernel mean matching} (KMM)~\citep{huang2007correcting}, the \emph{Kullback--Leibler importance estimation procedure} (KLIEP)~\citep{sugiyama2008direct}, and \emph{least-squares importance fitting} (LSIF)~\citep{kanamori2009least}, as shown in Table~\ref{tab:divergence_comparison}.}
Additional discussion of related approaches for learning under distribution shift is provided in Appendix~\ref{app:ds}.

\subsection{Dynamic Importance Weighting}
\label{sec:diw}
\emph{Dynamic importance weighting} (DIW) \citep{fang2020rethinking} was proposed as an end-to-end solution that adapts classic IW to deep learning.
In DIW, a feature extractor derived from the current deep model is used to apply a \emph{nonlinear transformation} $\pi:\cX\times\cY\to\calZ$ that maps data into a $d_\mathrm{r}$-dimensional vector space $\calZ$, where $d_\mathrm{r}\ll d$.
The transformed random variable, denoted by $\bz = \pi(\bx, y)$, is induced by $(\bx, y)$ and inherits its randomness. 
WE on $\bz$ is considerably easier than on $(\bx, y)$.
The goal of DIW is to learn a set of weights
$\mathcal{W}=\{w_i=w(\bz_i^\tr)\}_{i=1}^{n_\tr}$
such that
\begin{align*}
    \bE_{p_\te}\left[\ell(\mf(\bx), y)\right] = \bE_{p_\tr}\left[w(\bz)\ell(\mf(\bx), y)\right].
\end{align*}

In DIW, two practical choices of $\pi$ were proposed:

(i) \emph{Hidden-layer-output Transformation.}
First, we estimate a class-prior ratio $r_y^* = \pte(y)/\ptr(y)$.
Next, $\{(\bx_i^\tr, y_i^\tr)\}_{i=1}^{n_\tr}$ and $\{(\bx_i^\val, y_i^\val)\}_{i=1}^{n_\val}$ are partitioned by the class label $y$.
Finally, WE is performed separately for each class, i.e., $C$ times over $C$ class partitions, according to
\begin{align*}
\frac{\pte(\bx,y)}{\ptr(\bx,y)}
&= \frac{\pte(y)\cdot\pte(\bx\mid y)}{\ptr(y)\cdot\ptr(\bx\mid y)} \\
&= \frac{\pte(y)}{\ptr(y)} \cdot \frac{\pte(\bx\mid y)}{\ptr(\bx\mid y)}
= r_y^* \cdot \frac{\pte(\bz\mid y)}{\ptr(\bz\mid y)},
\end{align*}
where $\bz$ denotes the transformed data obtained from the hidden-layer-output transformation, assuming a fixed, deterministic, and invertible $\pi$ as in Fang et al.~\cite{fang2020rethinking}.

(ii) \emph{Loss-value Transformation}. 
$\pi$ transforms each sample into its one-dimensional loss value, i.e., $\pi:(\bx,y)\mapsto\ell(\mf(\bx),y)$.
DIW aims to learn weights such that, under the current model parameters $\theta_t$,
\begin{align*}
\frac{1}{n_\val}\sum_{i=1}^{n_\val}
\ell(\mf(\bx_i^\val),y_i^\val)\big|_{\theta=\theta_t}
&\approx \\
&\kern-0.6em
\frac{1}{n_\tr}\sum_{i=1}^{n_\tr}
w_i\,\ell(\mf(\bx_i^\tr),y_i^\tr)\big|_{\theta=\theta_t}.
\end{align*}
After applying $\pi$, KMM is performed on $\bz$ using a quadratic programming (QP) solver.
Once the weights are estimated, they are fixed, and the feature extractor and the deep model are updated by optimizing the WC objective.
By iterating between WE and WC within each mini-batch, DIW provides a modern and conceptually elegant realization of classic IW.

\section{Accelerated Dynamic Importance Weighting}
\label{sec:ADIW}
Although DIW adapts classic IW to deep learning, it remains limited in computational efficiency and support for diverse WE methods, as discussed in Section~\ref{sec:intro}.
To address these limitations, we now introduce \emph{accelerated dynamic importance weighting}~(ADIW), including its unified framework, efficient implementation, algorithm, and convergence analysis.

\subsection{A Unified IW Framework for Deep Learning}
\label{sec:ADIW_framework}
Let $\calP(\calZ)$ denote the set of probability distributions on $\calZ$ 
{that have densities.}
Following DIW \citep{fang2020rethinking}, $\pi$ is applied to transform the original joint densities $p_\tr(\bx, y)$ and $p_\te(\bx, y)$ into $p_\tr(\bz)$ and $p_\te(\bz)$, respectively.


For any two densities $p, q \in \calP(\calZ)$, a \emph{probability divergence} is defined as a function $D:\calP(\calZ)\times\calP(\calZ)\to\mathbbR$ satisfying: 

(i) Non-negativity: $D(p\parallel q)\geq0, \forall p, q\in\calP(\calZ)$,

(ii) Positivity: $D(p\parallel q)=0$ iff (i.e., if and only if) $p=q$.

Common families of divergences include \emph{integral probability metrics}~(IPMs) \citep{muller1997integral}, \emph{f-divergences} \citep{ali1966general}, and \emph{Bregman divergences} \citep{bregman1967relaxation},
which are briefly introduced below.

IPMs cover a wide range of well-known statistical distances, including the \emph{Wasserstein-1 distance} \citep{ln1969markov}, the \emph{total variation}~(TV) \emph{distance} \citep{gibbs2002choosing} (which is also an $f$-divergence), and the \emph{maximum mean discrepancy}~(MMD) \citep{borgwardt2006integrating,gretton2012kernel}.
Formally, given a class $\calF$ of real-valued functions on $\calZ$, an IPM between distributions $p$ and $q$ is defined as
\begin{align}
    D_{\calF}(p \parallel q) = \sup_{f\in \calF} \left| \int_\calZ f \dif p - \int_\calZ f \dif q \right|.
\end{align}
IPMs are generally \emph{pseudometrics} \citep{sriperumbudur2010hilbert} since most instances of $D_\calF$ do not satisfy condition~(ii).
However, some notable exceptions, including the Wasserstein-1 distance, the TV distance, and MMD with \emph{characteristic kernels} \citep{fukumizu2007kernel} (e.g., Gaussian and Laplacian kernels), satisfy both conditions~(i) and~(ii).

The family of $f$-divergences includes, but is not limited to, the \emph{Kullback–Leibler} (KL) \emph{divergence} \citep{kullback1951information}, the \emph{Hellinger distance} \citep{hellinger1909neue}, and the TV distance.
Assuming that $p$ is absolutely continuous with respect to $q$, the $f$-divergence between $p$ and $q$, induced by a convex function $f$, is defined as
\begin{align}
    D_{f}(p \parallel q) = \int_\calZ f \left(\frac{\dif p}{\dif q} \right)\dif q.
\end{align}
According to Khosravifard et al.~\cite{khosravifard2007confliction}, $D_{f}$ satisfies condition~(i) if and only if $f(1)\geq0$, and satisfies condition~(ii) if and only if $f(1)=0$ and $f$ is strictly convex at $x=1$.

The Bregman divergence is another important family of divergences, with the \emph{squared Euclidean distance} being a representative example.
Let $F: \calZ \to \mathbbR$ be a continuously differentiable and strictly convex function.
The Bregman divergence induced by $F$ between $p$ and $q$ is defined as
\begin{align}
    D_{F}(p \parallel q) = F(p) - F(q) - \left \langle \nabla F(q), p-q \right \rangle,
\end{align}
which satisfies both conditions~(i) and~(ii).

After applying $\pi$, WE is formulated as a density-matching problem between the transformed test density $p_\te(\bz)$ and the weighted transformed training density $w(\bz)\,p_\tr(\bz)$. 
Specifically, for $p_\te, \, w\cdot p_\tr \in \calP(\calZ)$,
the goal of WE is to find $w$ within a feasible set that minimizes a probability divergence $J(w)$ between $p_\te$ and $w \cdot p_\tr$:
\begin{align}
    \label{eq:ADIW_obj_exp}
    J(w) := D\left(p_\te \parallel w\cdot p_\tr\right),
\end{align}
where $w$ may represent either a weight vector or a parametric weight model.
Let $\widehat J(\calW)$ denote its empirical approximation.
Then, the empirical WE objective is
\begin{align}
    \label{eq:ADIW_obj_emp}
    \min_{\calW\in Q}\widehat J(\calW),
\end{align}
where $Q$ denotes the feasible constraint set for $\calW$.



\subsection{Efficient Implementation of ADIW}
Building upon the above unified framework, we now present the efficient implementation of ADIW.

In DIW, the WE optimization is solved to convergence at every iteration, which is often unnecessary for two reasons.
First, since WE operates on transformed representations, it only approximates true weights, making full convergence overly restrictive.
Second, WE and WC alternate at the mini-batch level: once WC updates the classifier, the transformed representations used by WE change accordingly, rendering the previous exact WE solution obsolete. 
As a result, an inexact but feasible update may be sufficient to ensure progress. 

These insights lead us to the design of ADIW, where WE is performed using a lightweight \emph{projected gradient descent} (PGD), which is also known as \emph{gradient projection} \citep{polyak1966constrained}. 
Specifically, if $\widehat J(\calW)$ is a differentiable and convex function with respect to $\calW$, we can update $\calW$ by a \emph{gradient descent} step followed by a projection onto the constraint set. 
Assume $\widehat J(\calW)$ is continuously differentiable and $Q$ is a convex and compact set encoding the constraints.
Then the empirical WE objective in Eq.~\eqref{eq:ADIW_obj_emp} can be solved as
\begin{align}
    \calW_{t+1} \gets \Pi_Q\left(\calW_{t} - \eta_t \nabla \widehat J(\calW_t)\right),
\end{align}
where $\calW_{t}$ is the set of weights at iteration $t$, $\eta_t\in(0, \infty)$ is the step size, and $\Pi_Q$ is the projection operator onto $Q$ to satisfy the constraints.
By performing only a few PGD-based updates per mini-batch, 
{
ADIW performs WE entirely on GPUs without QP solvers, greatly improving efficiency.}

\begin{algorithm}[t]
    \caption{Accelerated dynamic importance weighting.}
    \label{alg:ADIW}
    \begin{algorithmic}
        \REQUIRE training minibatch $\Str$ with index set $\mathcal{I}_{\Str}$, validation minibatch $\Sval$, current 
        model $\boldsymbol{f}_{\theta}$, current weight vector $\calW$,
        WE objective $\widehat J$ with constraint set $Q$,  step size $\eta_t$, number of WE iterations $T$
    \end{algorithmic}
    \begin{algorithmic}[1]
        \STATE \texttt{retrieve} initial weights $\calW_1 \gets \calW[\mathcal{I}_{\Str}]$
        \STATE \texttt{forward} the input parts of $\Str$ \& $\Sval$
        \STATE \texttt{compute} the loss values as $\Ltr$ \& $\Lval$
        \FOR{t = 1 to T}
        \STATE \texttt{compute} gradients $\nabla_{\calW_t} \widehat J$ from $\Ltr$ \& $\Lval$
        \STATE \texttt{update} $\calW_{t+1} \gets \calW_t - \eta_t \nabla_{\calW_t} \widehat J$
        \STATE \texttt{project} $\calW_{t+1}$ onto $Q$
        \ENDFOR
        \STATE \texttt{update} weights in $\calW$ for $\Str$: $\calW[\mathcal{I}_{\Str}] \gets \calW_T$
        \STATE \texttt{weight} the empirical risk $\widehat{R}(\mf)$ by $\mathcal{W}_T$ 
        \STATE \texttt{backward} $\widehat{R}(\mf)$ and \texttt{update} $\theta$
    \end{algorithmic}
\end{algorithm}

\subsection{The ADIW Algorithm}
\label{sec:algo}
The ADIW algorithm is summarized in Algorithm~\ref{alg:ADIW}, using
the loss-value transformation as $\pi$.
The variant based on the hidden-layer-output transformation is provided in Algorithm~\ref{alg:ADIW-F} in Appendix~\ref{sec:appendix-algo}.
For simplicity, we focus on the loss-value transformation in the main text, while Section~\ref{sec:ab_representation} presents an ablation study comparing the two transformations.
Although our discussion focuses on classification, the concept of ADIW may also be applied to other tasks (e.g., regression), provided that an appropriate transformation $\pi$ is designed for the task.

The above algorithm applies to non-parametric IW methods (e.g., those based on MMD or the Wasserstein-1 distance).
For parametric methods
(e.g., those based on the KL divergence or the squared distance), 
$\calW$ denotes the parameters of the weight model, whose dimensionality is determined by the validation set. 
In this case, the algorithm updates the weight model parameters rather than individual weights, using validation minibatch indices.
Section~\ref{sec:examples} presents several instantiations of classic IW methods within ADIW.

\begin{figure}[t]
     \centering
     \includegraphics[width=\columnwidth]{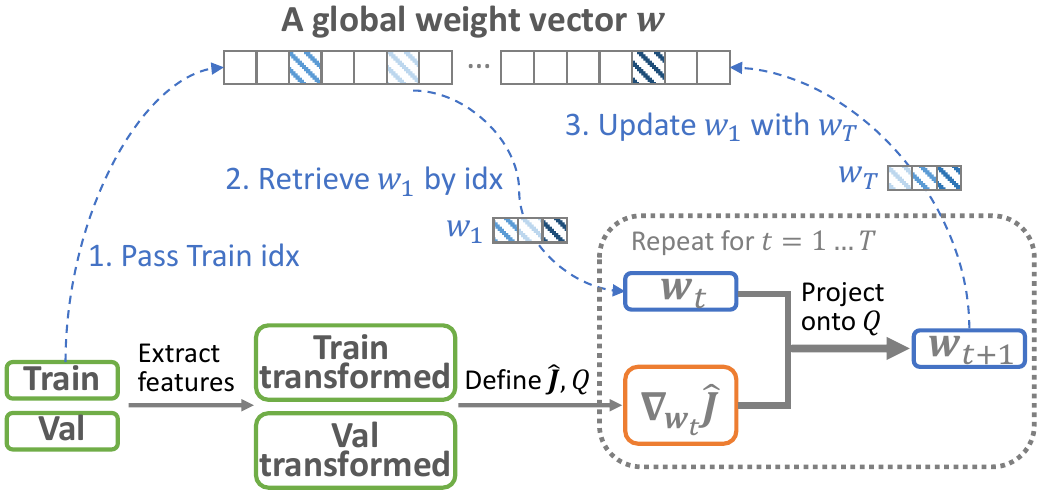}
     \caption{Illustration of WE in ADIW. ``idx'' denotes the index; ``Train'' and ``Val'' denote training and validation data, respectively.}
     \label{fig:adiw_we}
\end{figure}

An illustration of WE in ADIW is given in Figure~\ref{fig:adiw_we}.
ADIW maintains a global weight vector $\calW$, initialized to all ones in the first epoch and updated thereafter. 
{
In each iteration, mini-batches of training and validation data are sampled, and a subset $\calW_1$ corresponding to the sampled training indices is retrieved to initialize WE.
Starting from $\calW_1$, the WE loop performs $T$ steps of gradient descent on the transformed data, where each step is followed by a projection onto $Q$, yielding the updated weights $\calW_T$.
The resulting weights are then written back into the global vector $\calW$, so that subsequent mini-batches start from the most recently updated values.
The same weights $\calW_T$ are also used in the WC step to update the model parameters $\theta$.}

We note that ADIW defines the validation data loader, used to repeatedly sample validation mini-batches, once outside the training loop and recreates it only when exhausted~\footnote{{All methods in Figure~\ref{fig:overview}(\subref{fig:time}) adopt this improvement.}}, 
instead of reconstructing it at every iteration as in DIW.
This eliminates redundant loader construction,
which has little impact on small datasets but yields noticeable efficiency gains for large-scale datasets (see Section~\ref{sec:ab_profiling}).

{
To highlight the advantages of ADIW,
Table~\ref{tab:divergence_comparison} presents a systematic comparison of classic IW methods, DIW and ADIW.
Classic two-step IW methods are limited by their incompatibility with deep learning, whereas DIW mitigates this by performing WE on the transformed data in a mini-batch manner, but remains coverage restricted and computationally expensive.
ADIW overcomes these drawbacks by extending the framework to support diverse divergences and replacing the costly QP optimization with lightweight, GPU-accelerated PGD updates, thereby achieving both broader coverage and scalability for modern deep learning applications.}

\subsection{Convergence Analysis}
\label{sec:proof}%
Here we present the convergence analysis of ADIW.
We consider a slightly modified version of Algorithm~\ref{alg:ADIW}, where the stochastic gradient generator returns an unbiased estimate of the true gradient at each iteration.
Let the empirical weighted objective be 
\begin{align}
\label{empirical-weighted-training}
    R_\tr(\btheta,w)=\frac{1}{n_\tr}\sum_{i=1}^{n_\tr}w(\bz_i)\ell(\boldsymbol{f}_{\btheta}(\bx_i),y_i),
\end{align}
where each data point $(\bx_i,y_i)$ is mapped to a representation $\bz_i=\pi_{\btheta}(\bx_i,y_i)$ via a feature extractor $\pi_{\btheta}$. $w^*(\bz)=\frac{p_\te(\bz)}{p_\tr(\bz)}$ is the true importance weight, and $w(\bz)$ is its learned approximation via Eq.~\eqref{eq:ADIW_obj_emp}.
{
We next present the main convergence result, with the proof deferred to Appendix~\ref{sec:appendix_proof}.}

\begin{theorem}[Convergence of ADIW]
    \label{thm: convergence rate}
    Suppose Assumptions~\ref{as:differentiable}–\ref{as:Sensitivity} {in Appendix~\ref{app:assumptions}} hold. 
    Let the surrogate loss $\ell$ be bounded as $\ell< M$ and $M>0$, the learned weights satisfy $w\le B$, $B>0$, the number of inner-loop steps be $T$, and the number of outer-loop epochs be $S$.
    Set the outer-loop learning rate as $\alpha_s=\tfrac{c}{\sqrt{S}}$ for some constant $c>0$.
    Then, the output of ADIW after $S$ epochs satisfies     
    \begin{align}
    \bE[\|\nabla R_\tr(\btheta_s,\calW_{s})\|^{2}] \leq \frac{\Delta}{\sqrt{S}},
    \end{align}
    where $\|\cdot\|$ denotes $L_2$ norm, $\Delta=\frac{2BM}{c}+2cL\sigma^{2}+\frac{4cMT\sigma^{2}}{n_\tr}.$
    Equivalently, ADIW finds an $\epsilon$-stationary point, i.e., 
    $$\bE[\|\nabla R_\tr(\btheta_s,\calW_{s})\|^{2}]\leq\epsilon$$ within $O\left(1 / \epsilon^{2}\right)$ iterations.
\end{theorem}
The theorem establishes that ADIW converges to a stationary point of the weighted training objective in Eq.~\eqref{empirical-weighted-training} at the standard $O(1/\sqrt{S})$ rate. 
{The constant in the leading-order term includes}
a term that decreases with the training sample size $n_\tr$ and vanishes for large datasets. Consequently, for sufficiently large datasets, convergence is governed primarily by optimization complexity, whereas for smaller datasets statistical estimation error also plays a role.


\section{Instantiations of Classic IW Methods}
\label{sec:examples}
In this section, we illustrate how classic IW methods can be instantiated within the unified ADIW framework, each corresponding to a particular divergence.
We also introduce a novel IW variant based on the Wasserstein-1 distance.

\subsection{Maximum Mean Discrepancy}
We first consider ADIW with the \emph{maximum mean discrepancy}~(MMD), which recovers KMM~\citep{huang2007correcting}.
Let $\cH$ be a Hilbert space of real-valued functions $f: \mathbbR^{d_r}\to\mathbbR$ 
with inner product $\langle\cdot,\cdot\rangle_\cH$, 
or $\cH$ be a \emph{reproducing kernel Hilbert space}~(RKHS), where $k:(\bz,\bz')\mapsto\langle\phi(\bz),\phi(\bz')\rangle_\cH$ is the reproducing kernel of $\cH$ with a feature map $\phi:\bR^{d_\mathrm{r}}\to\cH$.
Then, the objective of weight estimation in Eq.~\eqref{eq:ADIW_obj_exp} with MMD is given by
\begin{align*}
D_\mathrm{MMD}\!\left(p_\te \parallel w\cdot p_\tr \right) = \sup_{\|f\|_\cH\le1}
\bE_{p_\te}[f(\bz)] - \bE_{p_\tr\cdot w(\bz)}[f(\bz)].
\end{align*}
Let $\widehat J_\mathrm{KM}$ denote the empirical objective of KMM, which minimizes the empirical estimate of the squared MMD with sample averages: 
\begin{align}
\label{eq:ADIW_kmm_emp}
\widehat J_\mathrm{KM}(\bw) = \: &\bw^\T\bK\bw - 2\bk^\T\bw \\
&\mathrm{subject~to~} \; w_i\ge0 \; \mathrm{and} \; |\frac{1}{n_\tr}\sum_{i=1}^{n_\tr}w_i-1|\le\epsilon, \notag
\end{align}
where $\bw\in\bR^{n_\tr}$ be the weight vector, $\bK_{ij}:=k(\bz_i^\tr,\bz_j^\tr)$, $\bk_i:=\frac{n_\tr}{n_\val}\sum_{j=1}^{n_\val}k(\bz_i^\tr,\bz_j^\val)$, and $\epsilon>0$ is a slack variable.

\subsection{Kullback--Leibler Divergence}
Next, we consider ADIW with the \emph{Kullback--Leibler} (KL) \emph{divergence}, as used in KLIEP~\citep{sugiyama2008direct}.
Let the weight $w(\bz)$ be modeled by a linear-in-parameter model 
$g(\bz)=\sum\nolimits_{l=1}^\nparams\beta_l \psi_l(\bz)=\bbeta^\top\boldsymbol{\psi}(\bz)$,
where $\bbeta=(\beta_1,\ldots,\beta_\nparams)^\top\in\bR^\nparams$ are learnable parameters and $\boldsymbol{\psi}(\bz)$ is a vector of $\nparams$ \emph{basis functions}. That is, $\boldsymbol{\psi}(\bz) = (\psi_1(\bz),\ldots,\psi_\nparams(\bz))^\top\in\bR^\nparams$, where 
$\psi_l(\bz):=\exp\!\left(-\frac{\|\bz-\bz_l^\val\|^2}{2\sigma^2}\right)$,
and $\{\bz_l^\val\}_{l=1}^{\nparams}$ are kernel centers from the validation data.
With the KL divergence, 
Eq.~\eqref{eq:ADIW_obj_exp} becomes
\begin{align*}
D_{\mathrm{KL}}\!\left(p_\te \parallel w\cdot p_\tr\right)
&= \bE_{p_\te}\!\left[\log \frac{p_\te(\bz)}{g(\bz)p_\tr(\bz)}\right] \\
&= \underbrace{\bE_{p_\te}\!\left[\log \frac{p_\te(\bz)}{p_\tr(\bz)}\right]}_{:=\mathrm{const.}}
   - \bE_{p_\te}\!\left[\log g(\bz)\right].
\end{align*}
After dropping the first constant term and plugging in the definition of $g(\bz)$, the empirical objective of KLIEP corresponds to the empirical counterpart of the second term, which is defined as
\begin{align}
\label{eq:ADIW_kliep_emp}
\widehat J_{\mathrm{KL}}(\bbeta)
&= - \frac{1}{n_\val}\sum_{i=1}^{n_\val}
   \log\bbeta^\top\boldsymbol{\psi}(\bz_i^\val) \notag \\
&\mathrm{~subject~to~}\quad
\bbeta^\top \boldsymbol{\bar{\psi}}_\tr = 1,\quad
\bbeta\geq\textbf{0}_\nparams,
\end{align}
where $\boldsymbol{\bar{\psi}}_\tr := \frac{1}{n_\tr}\sum_{j=1}^{n_\tr}\boldsymbol{\psi}(\bz_j^\tr)$
and $\textbf{0}_\nparams$ are the $\nparams$-dimensional vectors with all zeros. $\bbeta\geq\textbf{0}_\nparams$ is applied in the element-wise manner. 

\subsection{Squared Distance}
We then consider the case of the \emph{squared distance}, as used in \emph{least-squares importance fitting} (LSIF)~\citep{kanamori2009least}.
Let us model $w(\bz)$ by $g(\bz)$ in the same way as KLIEP.
Then, the objective is to minimize the squared distance between $w$ and $g$:
\begin{align*}
D_\mathrm{SQ}(w \parallel g)
&= \tfrac12 \bE_{p_\tr}\!\left[(w(\bz)-g(\bz))^2\right] \\
&= \tfrac12 \bE_{p_\tr}\!\left[g^2(\bz)\right]
   - \bE_{p_\te}\!\left[g(\bz)\right]
   + \underbrace{\tfrac12 \bE_{p_\tr}\!\left[w^2(\bz)\right]}_{:=\mathrm{const.}}.
\end{align*}
After approximating the first two terms with empirical averages and adding a regularizer $\textbf{1}^{\T}_\nparams\bbeta := \sum^\nparams_{l=1}|\beta_l|$, the empirical objective of LSIF is given by
\begin{align}
    \label{eq:ADIW_lsif_emp}
    \widehat J_\mathrm{LS}(\bbeta) 
    & = \frac{1}{2}\bbeta^{\T}\widehat{\boldsymbol{H}}\bbeta-\widehat{\boldsymbol{h}}^{\T}\bbeta+\lambda\textbf{1}^{\T}_\nparams\bbeta \, \notag \\
    &\mathrm{~subject~to~}\quad \bbeta\geq\textbf{0}_\nparams,
\end{align}
where $\widehat{\boldsymbol{H}}_{l, l'}=\frac{1}{n_\tr}\sum_{i=1}^{n_\tr}\psi_{l}(\bz_i^\tr)\psi_{l'}(\bz_i^\tr)$, $\widehat{\boldsymbol{h}}_{l}=\frac{1}{n_\val}\sum_{j=1}^{n_\val}\psi_{l}(\bz_j^\val)$, and $\lambda\geq0$ is a regularization parameter.
Note that for KLIEP and LSIF, the optimization is performed with respect to $\bbeta$, and the importance weights are subsequently computed as $\bbeta^\top\boldsymbol{\psi}(\bz)$.

\subsection{Wasserstein-1 Distance}
Finally, we introduce a novel IW variant based on the \emph{Wasserstein-1 distance} (also known as the \emph{Earth Mover’s Distance} \citep{rubner2000earth}).
For the test density $p_\te$ and the weighted training density $w\cdot p_\tr$ defined on a compact metric space, the Wasserstein-1 distance is defined as 
\begin{align*}
    D_\mathrm{W_1}\left(p_\te \parallel w\cdot p_\tr \right) = \inf_{\gamma \in \Pi(p_\te, w\cdot p_\tr)} \mathbb{E}_{(\bz, \bz') \sim \gamma} [\|\bz - \bz'\|],
\end{align*}
where $\Pi(p_\te, w\cdot p_\tr)$ denotes 
{
the set of all couplings $\gamma$ of $p_\te$ and $w\cdot p_\tr$, i.e., joint distributions whose marginals are $p_\te$ and $w\cdot p_\tr$.}
Its dual representation can be derived from the \emph{Kantorovich–Rubinshtein Theorem} \citep{villani2003topics}:
\begin{align*}
    D_\mathrm{W_1}\left(p_\te \parallel w\cdot p_\tr \right) 
    &=\sup_{\|\Phi\|_\mathrm{Lip} \leq 1} \mathbb{E}_{p_\te} [\Phi(\bz)] - \mathbb{E}_{w\cdot p_\tr} [\Phi(\bz)] \\
    &=\sup_{\|\Phi\|_\mathrm{Lip} \leq 1} \mathbb{E}_{p_\te} [\Phi(\bz)] - \mathbb{E}_{p_\tr} [w(\bz)\Phi(\bz)],
\end{align*}
where $\|\Phi\|_{\mathrm{Lip}}$ denotes the smallest constant $L$ (with $L=1$ for 1-Lipschitz functions) such that
$|\Phi(\bz) - \Phi(\bz^{\prime})| \le L \|\bz - \bz^{\prime}\|$ holds for all $\bz, \bz^{\prime} \in \mathbb{R}^{d_\mathrm{r}}$ and the supremum is taken over all 1-Lipschitz \textit{critic} functions $\Phi: \mathbb{R}^{d_\mathrm{r}} \to \mathbb{R}$.

In practice, the critic $\Phi$ is approximated by a parameterized family $\{\Phi_\nu\}_{\nu \in \mathcal{V}}$, with each $\Phi_\nu$ implemented as a neural network with parameters $\nu \in \mathcal{V}$.
The optimal critic $\Phi^*$ is obtained by solving
\vspace{-0.6ex}
\begin{align}
    \label{eq:ADIW_wa}
    \max_{\nu\in\mathcal{V}, \, \|\Phi_\nu\|_\mathrm{Lip} \leq 1}\mathbb{E}_{p_\te} [\Phi_\nu(\bz)] - \mathbb{E}_{w\cdot p_\tr} [\Phi_\nu(\bz)].
\end{align}
To minimize $D_\mathrm{W_1}$ with respect to the weights $w(\bz)$, we compute its gradient as
\vspace{-0.6ex}
\begin{align}
    \nabla_w D_\mathrm{W_1} = -\Phi^*(\bz) p_\tr(\bz).
\end{align}
Directly enforcing the Lipschitz constraint for neural networks is known to be challenging \citep{arjovsky2017wasserstein}. 
Fortunately, the optimal critic $\Phi^*$ satisfies the property of exhibiting a unit gradient norm almost everywhere under the optimal coupling $\pi$ between $p_\te$ and $w\cdot p_\tr$ \citep{gulrajani2017improved}.
Leveraging this property, we adopt the \emph{gradient penalty} approach \citep{gulrajani2017improved}, which encourages the gradient norm of the critic to stay close to one along paths between samples from $p_\te$ and $w\cdot p_\tr$.
The resulting objective is given by
\begin{align}
\label{eq:ADIW_wa_reg}
J_{W_1}(\nu)
&= \bE_{w\cdot p_\tr}[\Phi_\nu(\bz)]
- \bE_{p_\te}[\Phi_\nu(\bz)] \notag \\
&\hspace{3em} + \kappa\,\bE_{p'}\!\left[(\|\nabla_{\bz}\Phi_\nu(\bz)\|-1)^2\right],
\end{align}
where $p'$ denotes the distribution of points uniformly sampled
along straight-line interpolations between pairs of samples
$\bz \sim p_\te$ and $\bz' \sim w\cdot p_\tr$,
i.e., $\tilde{\bz} = \epsilon \bz + (1-\epsilon)\bz'$ with
$\epsilon \sim \mathrm{Uniform}(0,1)$.
The penalty coefficient $\kappa > 0$ controls the strength of regularization.
Before optimizing the full objective, we warm-start the critic $\Phi_\nu$ by training it with only the first two terms for a few iterations (i.e., without the gradient penalty), which provides a stable initialization.

We note that the optimization procedure in the previous examples typically involves bi-level optimization: an outer optimization that minimizes the weighted classification risk with respect to the classifier parameters $\boldsymbol{f}_{\theta}$, and an inner optimization that learns the importance weights $w(\bz)$. 
However, employing the Wasserstein-1 distance introduces an additional level of optimization complexity.
Specifically, within each update step of the weight estimation (inner optimization), another maximization problem (Eq.~\eqref{eq:ADIW_wa}) must be solved to obtain the optimal critic $\Phi^*$ for evaluating the Wasserstein-1 distance.
Nevertheless, the additional maximization is performed using only a few lightweight updates on mini-batch representations, resulting in computational cost comparable to other ADIW variants in practice (see Figure~\ref{fig:ADIW_time}).

\vspace{5pt}
\section{Experiments}
\label{sec:ADIW_exp}
This section presents an empirical evaluation of ADIW on benchmark datasets.
The ADIW with KMM, KLIEP, LSIF, and using Wasserstein-1 distance for weight estimation (WE) is named as ADIW-KM, ADIW-KL, ADIW-LS, and ADIW-WA respectively.

\subsection{Experimental Setup}
The experiments were conducted under two types of distribution shift, class-prior shift and label noise, using four benchmark datasets: \emph{Fashion-MNIST} \citep{xiao2017}, \emph{CIFAR-10}, \emph{CIFAR-100} \citep{krizhevsky2009learning}, and \emph{ImageNet-1K}, also known as \emph{ImageNet Large Scale Visual Recognition Challenge 2012 (ILSVRC 2012)} \citep{russakovsky2015imagenet}.
We adopted LeNet-5 \citep{lecun1998gradient} as the base model for Fashion-MNIST, trained by SGD \citep{robbins1951stochastic}, and ResNet-18 \citep{he2016deep} as the base model for CIFAR-10/100 and ImageNet-1K, trained by Adam \citep{kingma15iclr}. 
The baselines included: 
\begin{itemize}
    \item \emph{Val-Only} trains the model from scratch using only the limited validation data;
    \item \emph{Uniform} uses the weights of all ones to weight the training data;
    \item \emph{Random} uses random weights drawn from the \emph{rectified Gaussian distribution};
    \item \emph{L2R} is the learning to reweight examples method \citep{ren2018learning};
    \item \emph{MW-Net} is short for meta-weight-net \citep{shu2019meta}, a parametric version of L2R;
    \item \emph{DIW} is dynamic importance weighting \citep{fang2020rethinking}.
\end{itemize}
Additional details of the datasets, models, and experimental setups are provided in Appendix~\ref{sec:appendix_setup}. We next describe the specific setups for class-prior shift and label noise, respectively.

\subsubsection{Experiments under Class-prior Shift} 
To simulate class-prior shift, we modified the training set of Fashion-MNIST following Buda et al.~\cite{buda2018systematic}. 
All classes were randomly partitioned into majority and minority groups, with a fraction $\mu$ of the classes designated as minority. 
All classes within each group were assigned equal sample sizes, but the sample size per class in the minority group was reduced relative to that of the majority group, with the imbalance ratio defined as $\rho$. 
For validation set, we randomly sampled 10 instances per class from the training data. 
The original test set of Fashion-MNIST was left unchanged for evaluation.

\subsubsection{Experiments under Label Noise}
The label-noise experiments were conducted on Fashion-MNIST, CIFAR-10, CIFAR-100, and ImageNet-1K.
We imposed label noise to the training set while keeping the validation/test set intact. 
Two types of label noise were imposed on the training data: \emph{pair flip} \citep{han2018co}, where a label may flip to its neighborhood class with a probability, and \emph{symmetric flip} \citep{van2015learning}, where a label may flip to all other classes with equal probability (i.e., the \emph{noise rate}).
For ImageNet-1K, we trained the model from scratch on the original training set with synthetic label noise, randomly selected 20,000 training data as a clean validation set, and evaluated on the original ImageNet-1K validation set. 
For other datasets, we created the clean validation set by randomly sampling 1,000 data from the training set. 

\begin{table}[t]
\caption{
Accuracy (\%) and time on Fashion-MNIST under class-prior shift.
}
\label{tab:adiw_ci}
\centering
\resizebox{0.8\columnwidth}{!}{%
\begin{tabular}{l|cc|cc}
\toprule
\multirow{2}{*}{Method} & \multicolumn{2}{c|}{$\mu=0.5, \rho=50$} & \multicolumn{2}{c}{$\mu=0.5, \rho=100$} \\
\cmidrule{2-5}
 & Acc. (std) & Time & Acc. (std) & Time \\
\midrule
Val-Only    & 65.55 (1.08)          & 0.02  & 65.55 (1.08)          & 0.02  \\
DIW         & \textbf{78.09 (0.46)} & 22.97 & \textbf{74.19 (0.39)} & 17.88 \\
ADIW-KM     & \textbf{77.69 (0.22)} & 6.84  & \textbf{74.23 (0.52)} & 5.81 \\
ADIW-KL     & \textbf{77.86 (0.11)} & 8.39  & \textbf{74.25 (0.30)} & 5.83 \\
ADIW-LS     & \textbf{76.93 (1.25)} & 7.67  & \textbf{74.13 (0.55)} & 5.70 \\
ADIW-WA     & \textbf{77.37 (0.43)} & 7.64  & \textbf{73.14 (1.50)} & 7.79 \\
\bottomrule
\end{tabular}
}
\vspace{4pt}

\footnotesize
\raggedright
Time denotes the training time in seconds. Results are reported as mean accuracy (standard deviation) over the last ten epochs (3 trials).
Best and statistically comparable methods (\textit{t}-test at 1\% level) are shown in bold.
\end{table}

\subsection{Experimental Results}
\label{sec:performance}
In this section, we show the empirical results on benchmark datasets under class-prior shift and label noise, respectively. 

\subsubsection{Class-prior Shift Experiments} 
Table~\ref{tab:adiw_ci} reports mean accuracy and average per-epoch training time 
under two class-prior shift settings: $\rho=50$ and $\rho=100$. 
Across both settings, all ADIW variants achieve accuracy comparable to DIW while consistently requiring substantially less training time per epoch.

\begin{figure*}[t]
    \centering
    \begin{minipage}[c]{0.03\textwidth}~\end{minipage}
    \hspace{2pt}
    \begin{minipage}[c]{0.95\textwidth}
    \centering
    \begin{minipage}[c]{0.333\linewidth}\centering \hspace{8pt} 0.3 Pair \end{minipage}%
    \begin{minipage}[c]{0.333\linewidth}\centering \hspace{10pt} 0.4 Symmetric \end{minipage}%
    \begin{minipage}[c]{0.333\linewidth}\centering \hspace{10pt} 0.5 Symmetric \end{minipage}
    \end{minipage}\\
    \vspace{2pt}

    \begin{minipage}[c]{0.03\textwidth}\flushright \rotatebox{90}{\textit{Fashion-MNIST}} \end{minipage} 
    \hspace{2pt}
    \begin{minipage}[c]{0.95\textwidth}
        \includegraphics[width=0.333\linewidth]{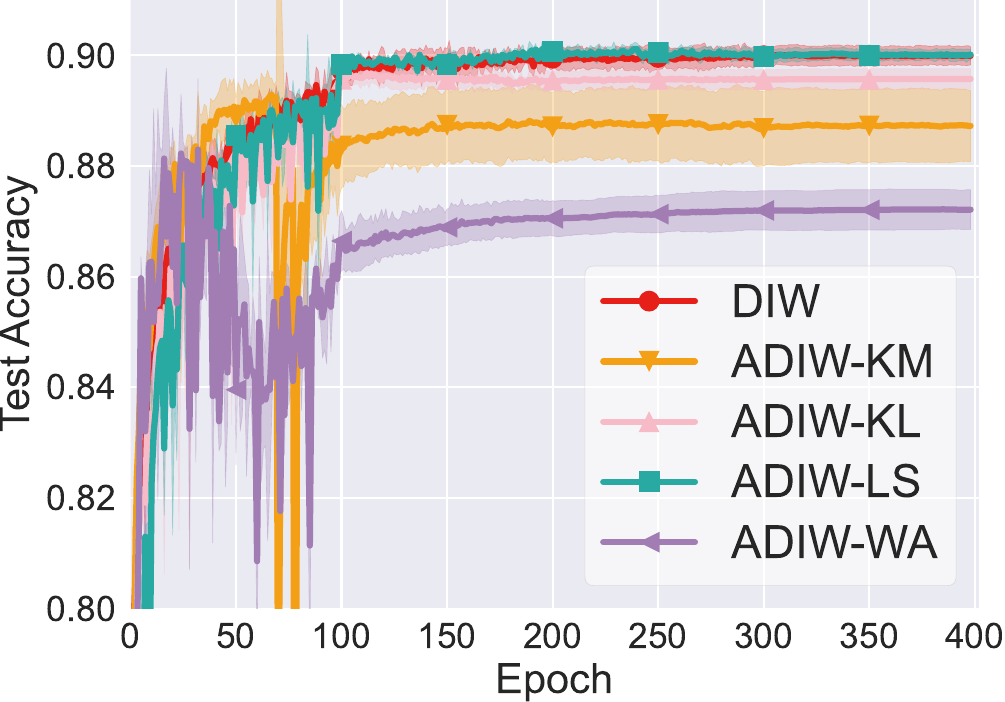}%
        \includegraphics[width=0.333\linewidth]{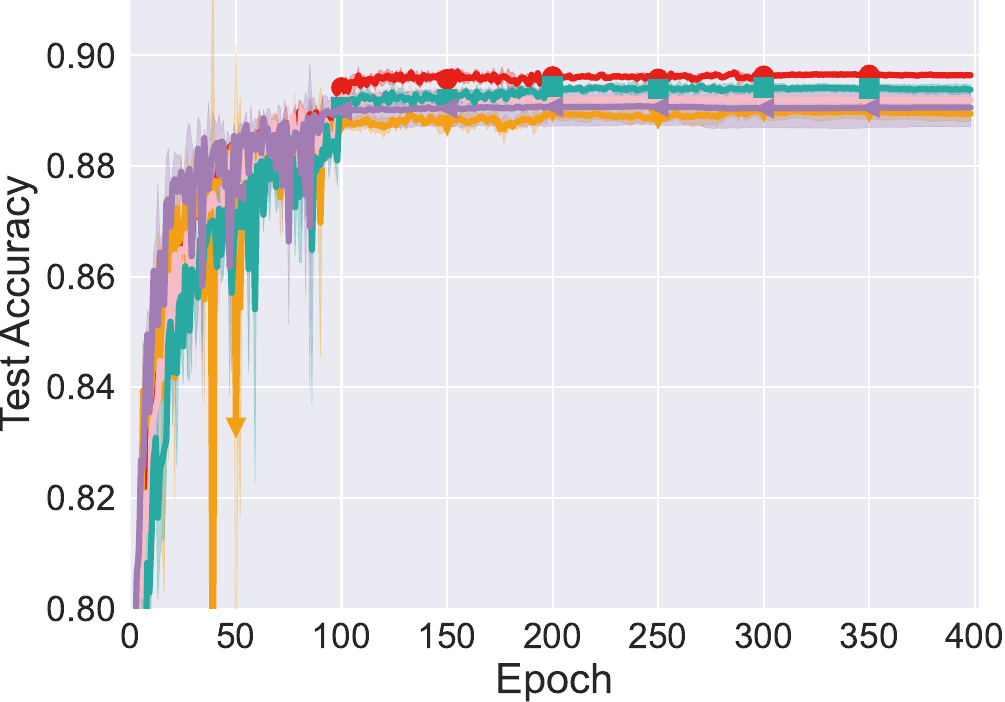}%
        \includegraphics[width=0.333\linewidth]{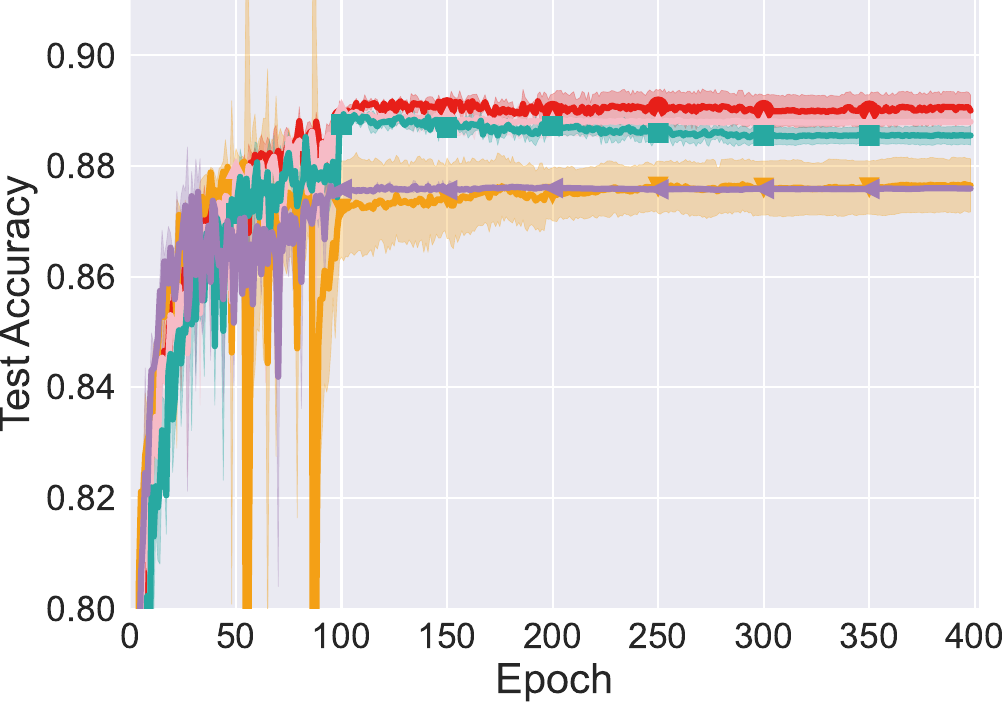}
    \end{minipage}\\
    \begin{minipage}[c]{0.03\textwidth}\flushright \rotatebox{90}{\textit{CIFAR-10}} \end{minipage}
    \hspace{2pt}
    \begin{minipage}[c]{0.95\textwidth}
        \includegraphics[width=0.333\linewidth]{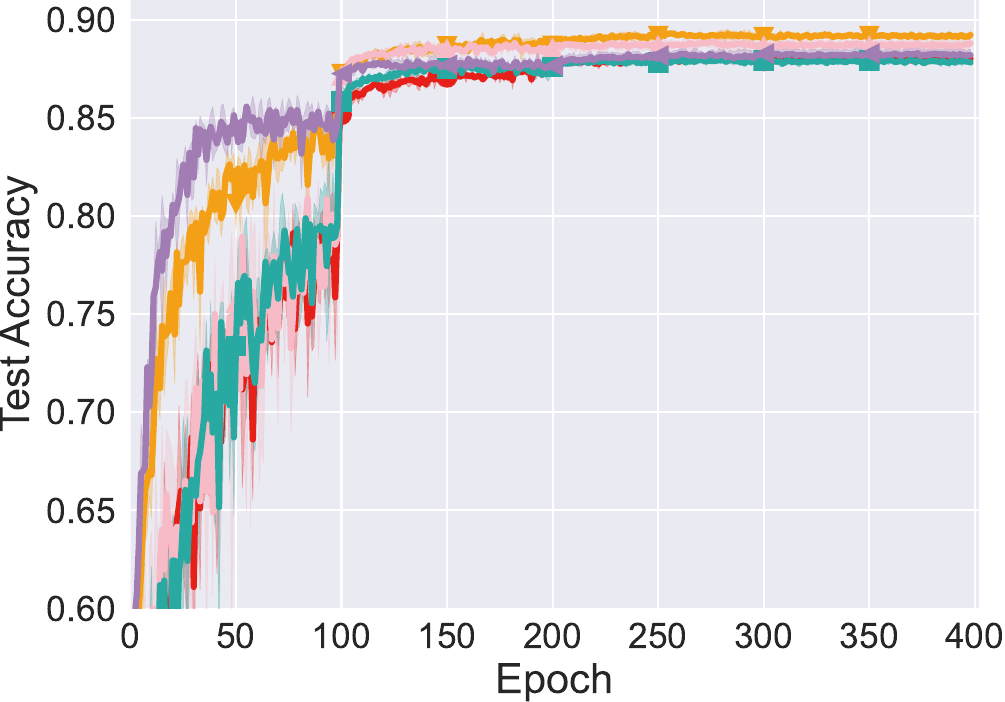}%
        \includegraphics[width=0.333\linewidth]{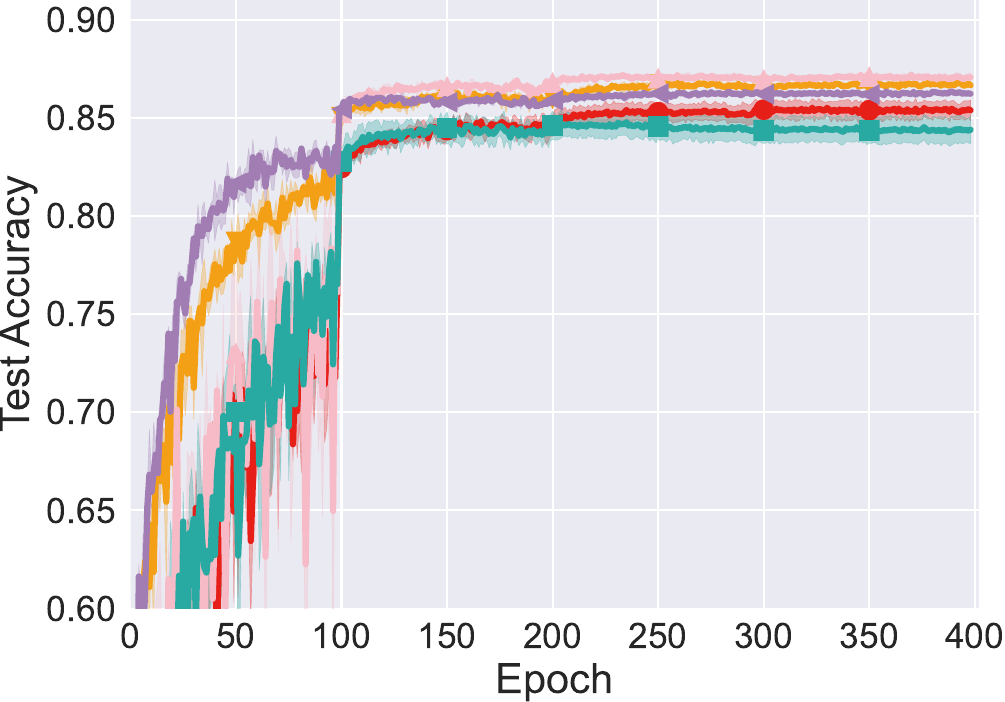}%
        \includegraphics[width=0.333\linewidth]{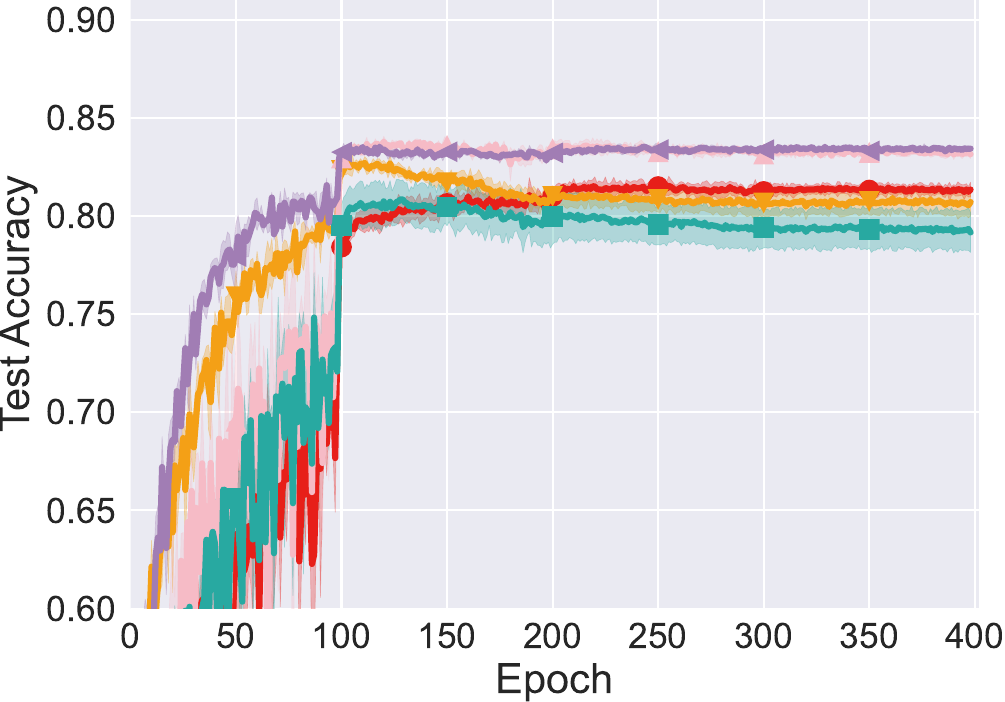}
    \end{minipage}\\
    \begin{minipage}[c]{0.03\textwidth}\flushright \rotatebox{90}{\textit{CIFAR-100}} \end{minipage}
    \hspace{2pt}
    \begin{minipage}[c]{0.95\textwidth}
        \includegraphics[width=0.333\linewidth]{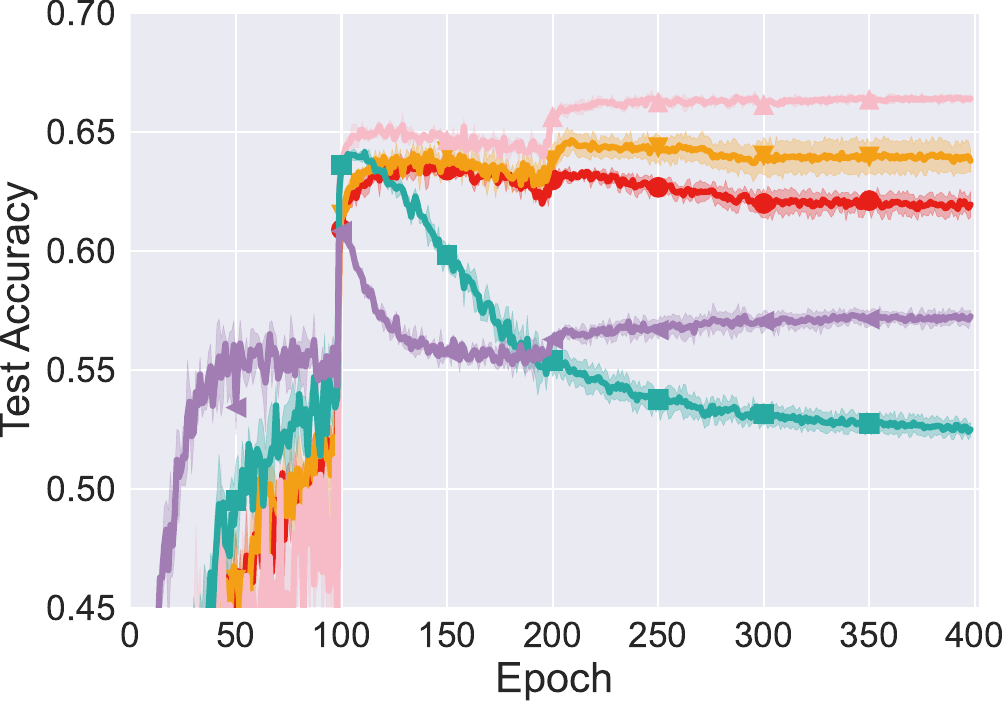}%
        \includegraphics[width=0.333\linewidth]{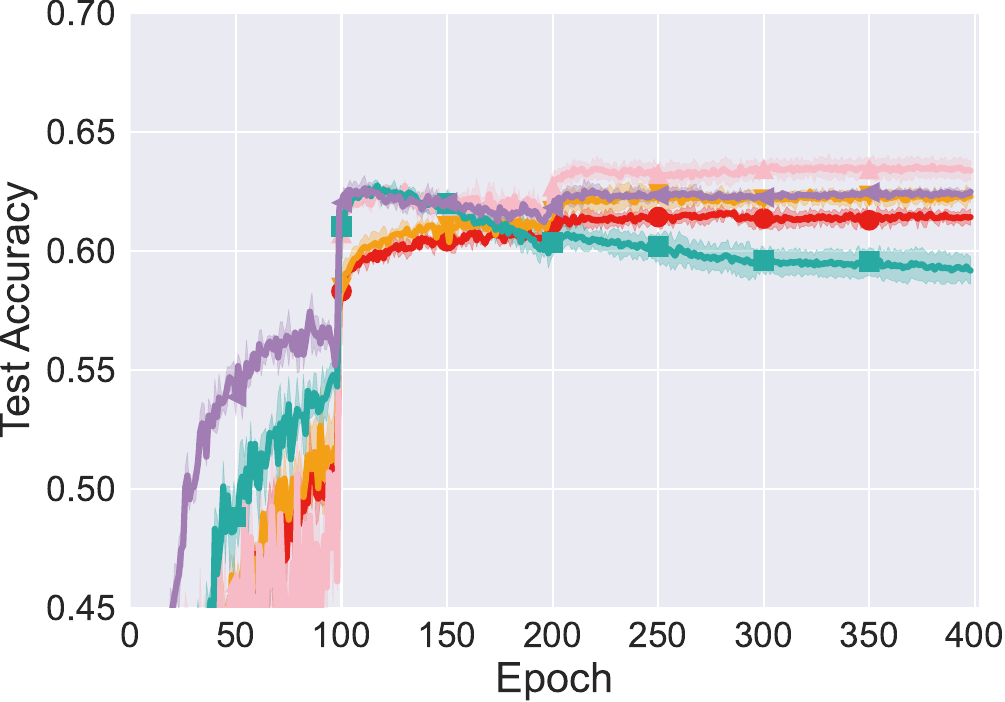}%
        \includegraphics[width=0.333\linewidth]{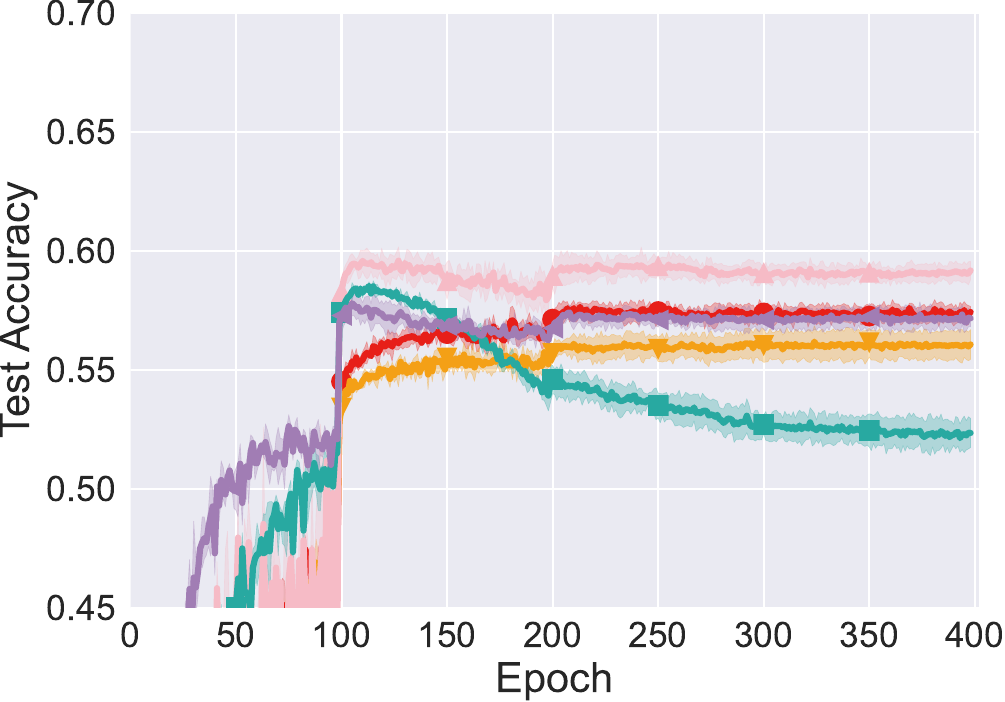}
    \end{minipage}
    \caption{
    Experimental results on Fashion-MNIST and CIFAR-10/100 under label noise (3 trials).
    }
    \label{fig:ADIW_performance_variant}
\end{figure*}

\begin{figure*}[t]
    \centering
    \makebox[\textwidth][c]{%
    \hspace{20pt}
    \begin{minipage}{0.245\textwidth}\centering \small Fashion-MNIST \end{minipage}%
    \begin{minipage}{0.245\textwidth}\centering \small CIFAR-10 \end{minipage}%
    \begin{minipage}{0.245\textwidth}\centering \small CIFAR-100 \end{minipage}%
    \begin{minipage}{0.245\textwidth}\centering \small ImageNet-1K \end{minipage}
    }

    \vspace{2pt}

    \makebox[\textwidth][c]{%
    \hspace{10pt}
    \begin{minipage}{0.245\textwidth}
        \includegraphics[width=\linewidth]{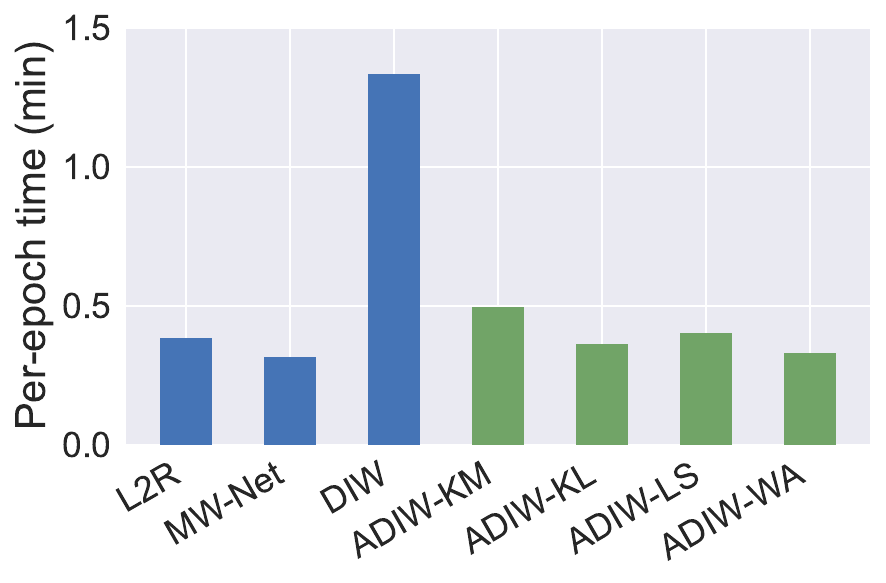}
    \end{minipage}%
    \begin{minipage}{0.245\textwidth}
        \includegraphics[width=\linewidth]{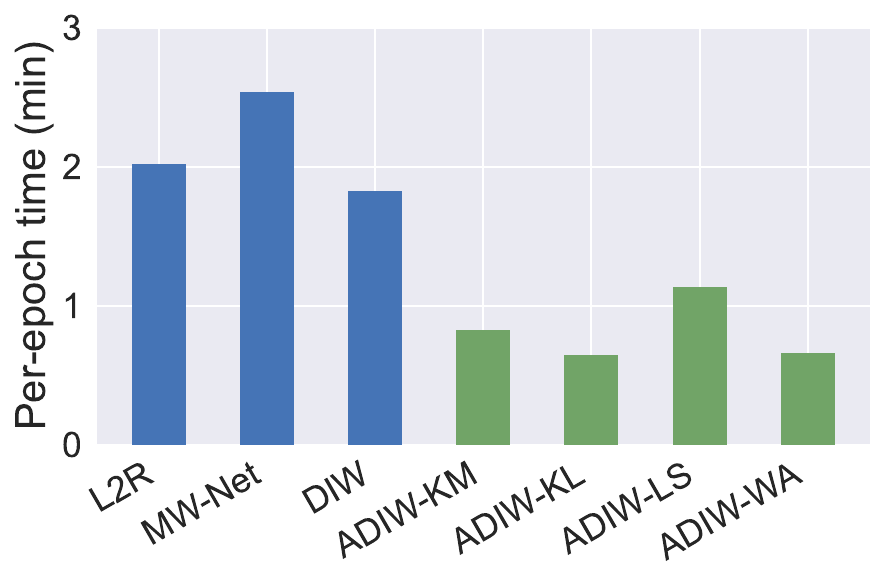}
    \end{minipage}%
    \begin{minipage}{0.245\textwidth}
        \includegraphics[width=\linewidth]{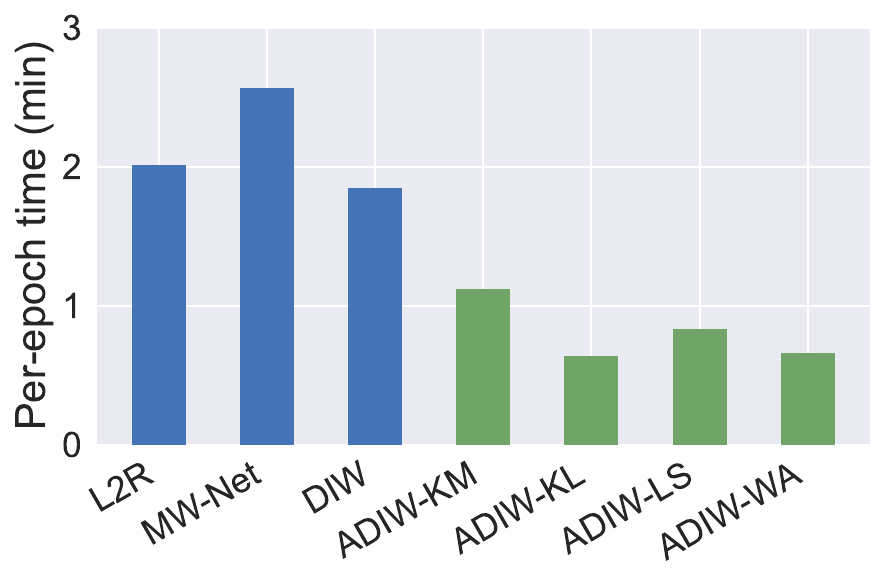}
    \end{minipage}%
    \begin{minipage}{0.245\textwidth}
        \includegraphics[width=\linewidth]{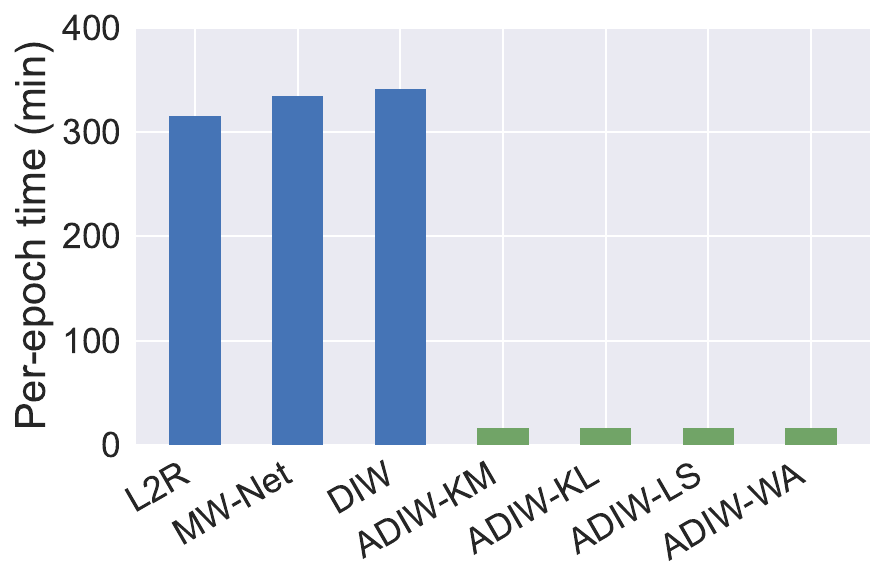}
    \end{minipage}
    }

    \caption{Training time (minutes) under 0.4 symmetric label noise.}
    \label{fig:ADIW_time}
\end{figure*}

\subsubsection{Label-noise Experiments}
Figure~\ref{fig:ADIW_performance_variant} shows classification accuracy of DIW and ADIW variants on Fashion-MNIST, CIFAR-10, and CIFAR-100 under 0.3 pair, 0.4 and 0.5 symmetric label noise, and a complete accuracy summary across all baselines is provided in Table~\ref{tab:ADIW_acc}.
Overall, the ADIW methods achieve performance that is competitive with, or superior to, DIW and other baselines.
Both DIW and ADIW also tend to be robust as the label noise rate increases.
An exception is that ADIW-WA underperforms under 0.3 pair-flip noise on Fashion-MNIST and CIFAR-100, and ADIW-LS exhibits lower robustness than other methods on the more challenging CIFAR-100 dataset.

\begin{figure}
    \centering
    \begin{minipage}{\columnwidth}
    \begin{minipage}[c]{0.55\columnwidth}\centering\small DIW \end{minipage}%
    \hspace{2pt}
    \begin{minipage}[c]{0.4\columnwidth}\centering\small ADIW-KL \end{minipage}%
    \\
    \begin{minipage}[c]{\columnwidth}
        \hspace{5pt}
        \includegraphics[width=0.47\columnwidth]{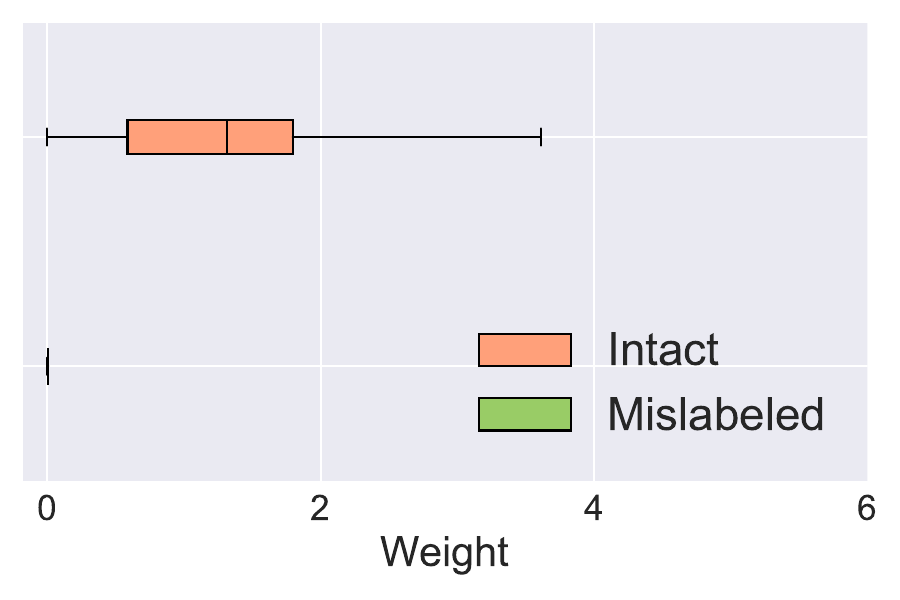}
        \hspace{3pt}
        \includegraphics[width=0.47\columnwidth]{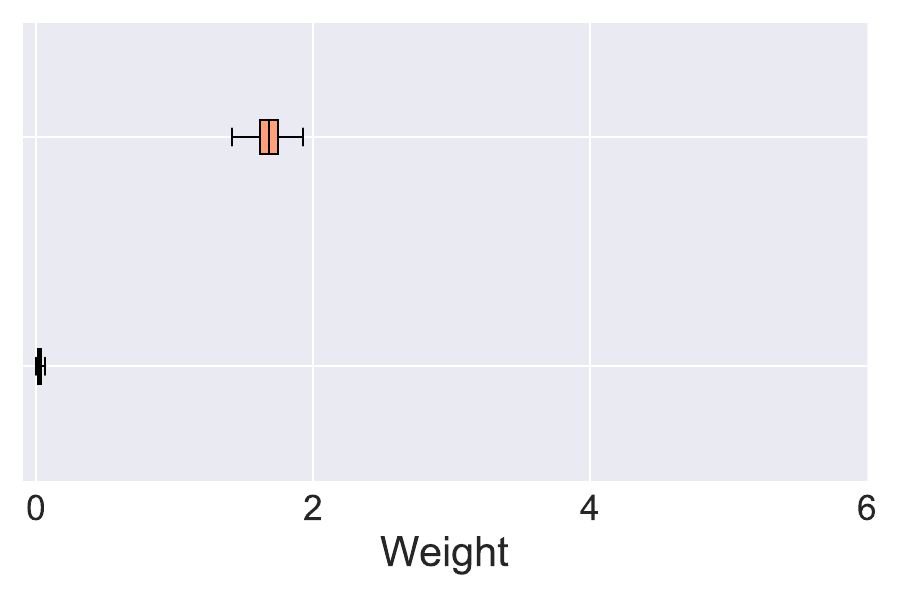}
    \end{minipage}\\
    \begin{minipage}[c]{\columnwidth}
        \includegraphics[width=0.495\columnwidth]{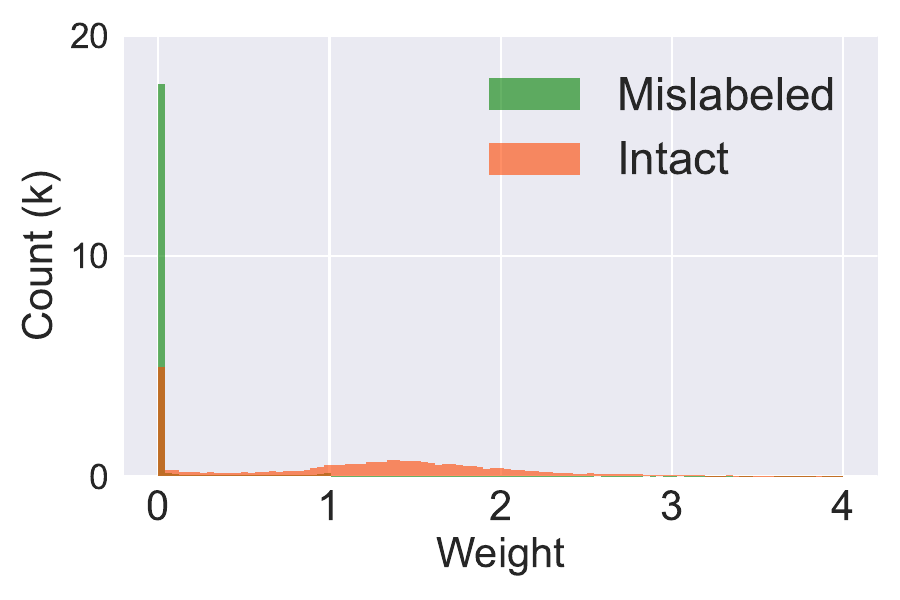}
        \includegraphics[width=0.495\columnwidth]{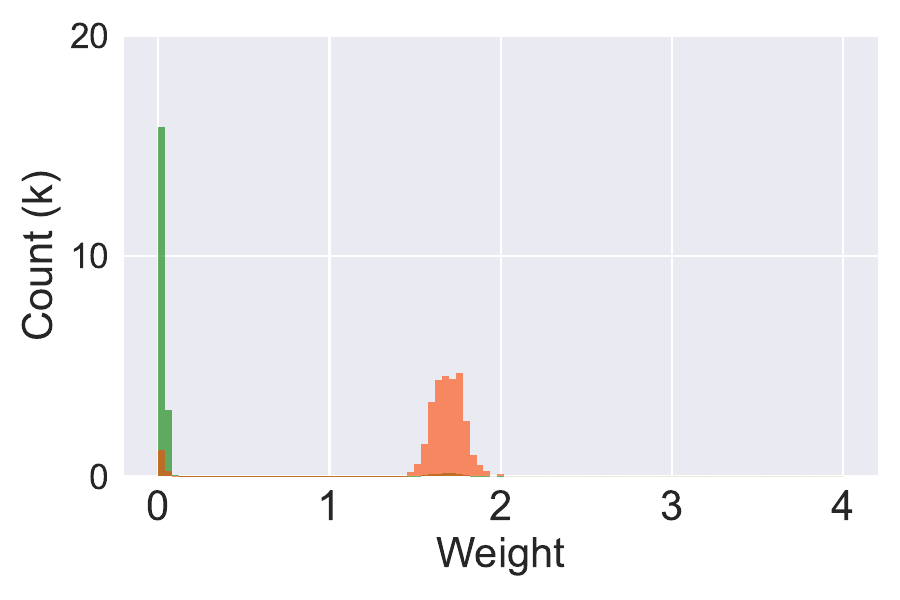}
    \end{minipage}\\
    \caption{Statistics of weight distributions on CIFAR-10 under 0.4 symmetric label noise (top: box plots; bottom: histogram plots).}
    \label{fig:ADIW_wdist_c10_sota}
\end{minipage}
\end{figure}

\begin{table}
\caption{Top-1/5 accuracy (\%) on ImageNet-1K under label noise.}
\resizebox{0.99\columnwidth}{!}{%
\label{tab:imagenet}
\begin{tabular}{c|ccccccccc}
    \toprule
    & Val-Only & Uniform & ADIW-KM & ADIW-KL & ADIW-LS & ADIW-WA \\
    \midrule
    Top-1 & \makecell{11.60 \\ (0.03)} & \makecell{57.66 \\ (0.20)} & \makecell{64.73 \\ (0.08)} & \makecell{\textbf{65.86} \\ \textbf{(0.16)}} & \makecell{64.16 \\ (0.07)} & \makecell{63.29 \\ (0.12)} \\
    Top-5 & \makecell{25.77 \\ (0.27)} & \makecell{78.89 \\ (0.10)} & \makecell{\textbf{84.66} \\ \textbf{(0.03)}} & \makecell{\textbf{84.82} \\ \textbf{(0.04)}} & \makecell{\textbf{84.65} \\ \textbf{(0.08)}} & \makecell{\textbf{83.59} \\ \textbf{(0.18)}} \\
    \bottomrule
\end{tabular}}

\vspace{4pt}

\footnotesize
\raggedright
Results are reported as mean accuracy (standard deviation) over the last ten epochs (3 trials).
Best and statistically comparable methods (\textit{t}-test at the 1\% significance level) are shown in bold.

\textit{Note:}
L2R, MW-Net, and DIW are computationally expensive and unable to complete training on ImageNet-1K.
Their training times are reported separately in Figure~\ref{fig:ADIW_time} and Table~\ref{tab:efficiency} in Appendix~\ref{app:time}.

\end{table}

{
Table~\ref{tab:imagenet} reports results on ImageNet-1K validation data under 0.4 symmetric label noise.
Baselines relying on differentiating through model parameters or quadratic programming solvers
(i.e., L2R, MW-Net, and DIW) are computationally infeasible at this scale and are thus omitted.
In contrast, ADIW enables efficient end-to-end GPU training and substantially outperforms naive baselines (i.e., Val-Only and Uniform). 
For ADIW, ADIW-KL achieves the best top-1 accuracy, while all variants achieve comparable top-5 accuracy.}

\begin{table*}
\caption{Stage-wise breakdown of total CPU and CUDA time on ImageNet-1K (seconds; percentages in parentheses).}
\label{tab:profiling_imagenet}
\centering
\resizebox{\textwidth}{!}{%
\begin{tabular}{l|ll|ll|ll|ll}
    \toprule
    \multirow{2}{*}{Stage} & \multicolumn{2}{c|}{DIW} & \multicolumn{2}{c|}{DIW\_Val} & \multicolumn{2}{c|}{ADIW\_CPU} & \multicolumn{2}{c}{ADIW\_GPU}\\
    \cmidrule(lr){2-3} \cmidrule(lr){4-5} \cmidrule(lr){6-7} \cmidrule(lr){8-9}
    & CPU (\%) & CUDA (\%) & CPU (\%) & CUDA (\%) & CPU (\%) & CUDA (\%) & CPU (\%) & CUDA (\%) \\
    \midrule
    Fetch tr data & 0.77 (1.02) & 0.16 (0.47) & 0.79 (3.47) & 0.14 (0.63) & 0.80 (10.65) & 0.16 (2.04) & 2.46 (71.96) & 0.17 (8.20) \\
    Fetch val data & 43.31 (57.07) & 0.14 (0.43) & 0.30 (1.31) & 0.14 (0.62) & 0.29 (3.83) & 0.15 (1.97) & 0.28 (8.09) & 0.16 (7.93) \\
    Forward tr data & 0.07 (0.09) & 0.57 (1.72) & 0.05 (0.22) & 0.54 (2.35) & 0.06 (0.82) & 0.52 (6.60) & 0.09 (2.52) & 0.52 (25.70) \\
    Forward val data & 0.09 (0.11) & 0.59 (1.78) & 0.06 (0.25) & 0.54 (2.32) & 0.06 (0.78) & 0.51 (6.57) & 0.09 (2.74) & 0.52 (25.66) \\
    Get tr loss & 0.01 (0.02) & 0.00 (0.00) & 0.00 (0.02) & 0.00 (0.00) & 0.00 (0.05) & 0.00 (0.00) & 0.00 (0.13) & 0.00 (0.01) \\
    Get val loss & 0.01 (0.02) & 0.00 (0.00) & 0.00 (0.02) & 0.00 (0.00) & 0.00 (0.05) & 0.00 (0.00) & 0.00	(0.12) & 0.00 (0.01) \\
    \rowcolor{lightgray}
    Estimate weights & 31.18 (41.07) & 30.97 (93.44) & 21.32 (93.28) & 21.11 (91.22) & 6.06 (80.56) & 5.86 (75.03) & 0.21 (6.22) & 0.04 (2.17) \\
    Weight tr loss & 0.21 (0.28) & 0.70 (2.12) & 0.13 (0.58) & 0.66 (2.83) & 0.09 (1.25) & 0.60 (7.68) & 0.10 (2.81) & 0.61 (29.89) \\
    Backward data & 0.20 (0.27) & 0.00 (0.00) & 0.17 (0.74) & 0.00 (0.00) & 0.12 (1.60) & 0.00 (0.00) & 0.15 (4.44) & 0.00 (0.00) \\
    Update model & 0.04 (0.06) & 0.01 (0.03) & 0.03 (0.13) & 0.01 (0.04) & 0.03 (0.41) & 0.01 (0.11) & 0.03 (0.95) & 0.01 (0.42) \\
    \midrule
    Total & 75.90 (100.00) & 33.14 (100.00) & 22.86 (100.00) & 23.15 (100.00) & 7.52 (100.00) & 7.81 (100.00) & 3.42 (100.00) & 2.04 (100.00) \\
    \bottomrule
\end{tabular}}

\vspace{4pt}

\footnotesize
\raggedright
The results are averaged over four profiling windows at epochs 1, 5, 10 and 20. Each profiling window consists of 10 iterations measured by PyTorch Profiler. ``tr'' and ``val'' denote training and validation data. The ``Estimate weights'' stage is highlighted as the primary target for efficiency improvement.

\end{table*}

In addition, Figure~\ref{fig:ADIW_time} visualizes per-epoch end-to-end training time under 0.4 symmetric label noise, with complete results under 0.3 pair and 0.5 symmetric label noise reported in Table~\ref{tab:efficiency} in Appendix~\ref{app:time}.
On the small-scale datasets (Fashion-MNIST, CIFAR-10, and CIFAR-100), ADIW reduces per-epoch training time by approximately 20–75\% compared to DIW.
The efficiency gain of ADIW is even more pronounced on ImageNet-1K, where ADIW achieves over 95\% reduction.
These results demonstrate that ADIW scales efficiently from small to large datasets, whereas existing reweighting-based methods such as L2R, MW-Net, and DIW become less practical on the large-scale datasets.

Finally, we visualized the learned weight distributions on noisy-labeled CIFAR-10 in Figures~\ref{fig:ADIW_wdist_c10_sota},
where the weights of intact samples are shown in orange and those of mislabeled samples in green. 
Both DIW and ADIW-KL assign near-zero weights to mislabeled data; however, ADIW-KL yields lower variance among intact-sample weights than DIW, indicating more stable weighting.
Results for all ADIW variants are provided in Figure~\ref{fig:ADIW_wdist_c10_sota_full} in Appendix~\ref{app:w_dist_full}.


\subsection{Ablation Study}
To gain deeper insight into the design and behavior of ADIW, we conducted several ablation studies.
First, we performed fine-grained profiling of multiple implementations to trace the evolution from DIW to ADIW and quantify the resulting efficiency gains.
{
Second, we investigated the effect of the number of validation samples used in the weight estimation.}
Finally, we examined the impact of adopting different data transformation choices within ADIW.

\subsubsection{Key-stage Profiling from DIW to ADIW}
\label{sec:ab_profiling}
To examine the efficiency improvements from DIW to ADIW, we profiled different implementations across key stages using the PyTorch Profiler (see Appendix~\ref{app:ab} for configuration details).
We considered the following four implementations:
\begin{itemize}
\item DIW: the original DIW implementation;
\item DIW\_Val: DIW with the validation loader defined outside the training loop as described in Section~\ref{sec:algo};
\item ADIW\_CPU: PGD-based WE performed on CPU;
\item ADIW\_GPU: PGD-based WE performed on GPU (i.e., the proposed method).
\end{itemize}

Table~\ref{tab:profiling_imagenet} presents stage-wise profiling results on ImageNet-1K under 0.4 symmetric label noise. 
For DIW, validation data fetching and WE are the main bottlenecks. 
DIW\_Val substantially reduces the cost of validation data fetching, but WE remains the dominant contributor to both CPU and CUDA time.
ADIW\_CPU replaces solving the quadratic program to convergence with lightweight PGD-based WE on CPU, achieving about 72\% reductions in both CPU and CUDA time, though WE remains the main bottleneck.
ADIW\_GPU performs PGD-based WE entirely on GPU, reducing total CPU and CUDA time by approximately 95\% and 94\%, respectively (compared with DIW), with WE contributing only 6\% of CPU time and 2\% of CUDA time.

{
These results demonstrate how ADIW\_GPU evolves from DIW through progressively improved designs toward fully GPU-accelerated and scalable training.}
Additional profiling results on CIFAR-10 are reported in Appendix~\ref{sec:profiling_c10}.

\subsubsection{Impact of Validation Sample Size}
\label{sec:ab_validation_size}
{
We conducted experiments on CIFAR-10 under 0.4 symmetric label noise with different validation set sizes (3 trials).
Results in Table~\ref{tab:val_size} show that ADIW-KL and ADIW-WA remain highly robust to the validation set size, with only minor performance variations when reducing the number of validation samples from 1000 to 10. 
In contrast, DIW exhibits noticeable performance degradation under extremely small validation sets, e.g., when reducing the validation size from 100 to 10 samples. 
Overall, ADIW variants tend to be more robust to limited validation data, although ADIW-LS shows relatively larger degradation while still outperforming DIW. 
}

\subsubsection{Comparison of Data Transformations}
\label{sec:ab_representation}
As discussed in Section~\ref{sec:algo}, the data transformation in ADIW can be either hidden-layer-output transformation (-F) or loss-value transformation (-L). 
We compared their performance on noisy-labeled CIFAR-10 in Table~\ref{tab:adiw_ln_rep},
where the time denotes the average end-to-end training time per epoch in seconds.
Overall, the -L transformation achieves strong performance across all divergence choices, providing computational efficiency and a simpler algorithmic design.
These advantages make it the preferred configuration adopted in all preceding experiments.


\begin{table}[t]
\caption{Effect of validation set size on CIFAR-10 under label noise.}
\label{tab:val_size}
\centering
\resizebox{\columnwidth}{!}{%
\begin{tabular}{l|ccccc}
\toprule
Size & DIW & ADIW-KM & ADIW-KL & ADIW-LS & ADIW-WA \\
\midrule
10   & 61.04 (1.10) & 82.06 (1.87) & 86.14 (0.40) & 65.00 (1.29) & 86.32 (0.22) \\
100  & 83.13 (0.56) & 86.29 (0.67) & 86.85 (0.29) & 82.82 (1.02) & 86.52 (0.08) \\
1000 & 85.38 (0.45) & 86.70 (0.06) & 87.05 (0.18) & 84.38 (0.68) & 86.27 (0.06) \\
\bottomrule
\end{tabular}}

\end{table}

\begin{table}[t]
\caption{Comparison of data transformation choices in ADIW on CIFAR-10.}
\label{tab:adiw_ln_rep}
\centering
\resizebox{\columnwidth}{!}{%
\begin{tabular}{l|l|cc|cc|cc}
\toprule
\multirow{2}{*}{Trans.} & \multirow{2}{*}{ADIW} & \multicolumn{2}{c|}{0.3 pair} & \multicolumn{2}{c|}{0.4 symmetric} & \multicolumn{2}{c}{0.5 symmetric} \\
\cmidrule{3-8}
& & Acc. (std) & Time & Acc. (std) & Time & Acc. (std) & Time \\
\midrule
\multirow{4}{*}{-F}
& -KM & \textbf{88.57 (0.17)} & 42.21 & 81.30 (0.21) & 42.40 & 75.37 (0.74) & 42.36 \\
& -KL & \textbf{87.92 (0.14)} & 42.24 & \textbf{85.50 (0.33)} & 41.37 & \textbf{81.31 (0.54)} & 42.21 \\
& -LS & \textbf{87.42 (0.27)} & 42.46 & 81.30 (0.31) & 42.42 & 75.94 (0.12) & 42.55 \\
& -WA & 67.54 (0.40) & 75.74 & 63.09 (0.68) & 75.79 & 56.26 (0.55) & 75.80 \\
\midrule
\multirow{4}{*}{-L}
& -KM & \textbf{89.16 (0.15)} & 55.04 & \textbf{86.70 (0.06)} & 49.71 & \textbf{80.64 (0.67)} & 38.73 \\
& -KL & \textbf{88.71 (0.27)} & 38.48 & \textbf{87.05 (0.18)} & 38.82 & \textbf{83.20 (0.20)} & 38.61 \\
& -LS & \textbf{87.86 (0.13)} & 69.59 & \textbf{84.38 (0.68)} & 67.94 & \textbf{79.26 (1.03)} & 68.53 \\
& -WA & \textbf{88.20 (0.30)} & 39.59 & \textbf{86.27 (0.06)} & 39.66 & \textbf{83.41 (0.12)} & 39.79 \\
\bottomrule
\end{tabular}}

\vspace{4pt}

\footnotesize
\raggedright
``Trans.'' denotes the data transformation type, where ``-F'' and ``-L'' indicate hidden-layer-output and loss-value transformations, respectively. The results are reported as mean accuracy (standard deviation) over the last ten epochs (3 trials).
Best and statistically comparable methods in accuracy (\textit{t}-test at the 1\% significance level) are shown in bold.

\end{table}

\section{Conclusions}
\label{sec:conclusion}
We analyzed the limitations of existing IW-based methods for joint distribution shift and proposed ADIW, a unified and efficient IW framework for deep learning. 
{By formulating weight estimation as a divergence-minimization problem, ADIW supports diverse weight estimation methods in a plug-and-play manner within a single implementation while improving efficiency through projected gradient descent updates.
Experiments demonstrated the effectiveness, scalability, and versatility of ADIW in handling distribution shifts.
}

Future work includes deeper theoretical analysis of the impact of different divergence measures within ADIW, as well as empirical investigations of how they affect performance across tasks such as reinforcement learning, time-series forecasting, and natural language processing. 
Additionally, 
future work may consider handling gradual distribution shifts, such as \emph{concept drift} \citep{lu2018learning}, which involve multiple sequential shifts rather than a single training-test distribution pair.

\bibliographystyle{IEEEtran}
\bibliography{example_paper}


 





\clearpage
\onecolumn
\appendices
\section{Discussion of IW and DA Approaches under Distribution Shift}
\label{app:ds}
This section briefly reviews distribution shift and representative approaches for handling it, including importance weighting and domain adaptation.

\textbf{Distribution Shift.} 
Joint distribution shift is the most general form of distribution shift, where the joint distribution $p(\bx,y)$ changes between the training and test data without additional assumptions, i.e., 
$\ptr(\bx,y)\neq\pte(\bx,y)$.
It can be specialized into the following cases under further assumptions \citep{Lu2023}:
\begin{itemize}
    \item Covariate shift: $p(\bx)$ changes while $p(y\mid \bx)$ remains unchanged;
    \item Class-prior shift: $p(y)$ changes while $p(x\mid y)$ remains unchanged;
    \item Class-posterior shift (i.e., label noise): $p(y\mid \bx)$ changes while $p(\bx)$ remains unchanged;
    \item Class-conditional shift: $p(\bx\mid y)$ changes while $p(y)$ is fixed. 
\end{itemize}

\textbf{Importance Weighting (IW).} 
As introduced in the main text, IW and IW-based methods can address the general joint distribution shift without assuming a specific shift form a priori. 
They typically rely on a small labeled validation set from the test distribution to estimate importance weights for the training data. 
Moreover, IW-based methods assume that $p_\te(\bx, y)$ is absolutely continuous with respect to $p_\tr(\bx, y)$, i.e., $p_\te(\bx, y) > 0$ implies $p_\tr(\bx, y) > 0$. 
Consequently, IW methods are usually evaluated under controlled distribution shifts, where reliable estimation of importance weights is feasible.

\textbf{Domain Adaptation (DA).} 
DA takes a different approach from IW and can be broadly categorized into supervised DA (SDA) and unsupervised DA (UDA) \citep{pan2009survey,wang2018deep}.
In SDA, a small amount of labeled target-domain data is available, whereas UDA assumes access to abundant unlabeled target-domain data.
While SDA is closer to the setting of IW, much of the DA literature focuses on UDA, where the objective is to reduce domain discrepancy through representation alignment \citep{long2015learning,ganin2016domain,tzeng2017adversarial} or pseudo-label-based transfer \citep{lee2013pseudo,saito2017asymmetric}.
These approaches typically rely on implicit assumptions, such as limited change in the conditional distribution, and therefore address only a restricted subset of joint distribution shift.
When such assumptions are violated, e.g., under severe label noise, DA methods often become less effective.
Overall, DA is primarily designed for more challenging cross-domain shifts with substantial domain discrepancy.

In summary, IW and DA tackle distribution shift under fundamentally different assumptions, problem setups, and objectives.
IW focuses on addressing general joint distribution shift by reweighting training data using a small labeled validation set under controlled shifts, whereas DA typically addresses a restricted subset of joint shift using abundant unlabeled target-domain data to learn transferable or domain-invariant representations for large cross-domain discrepancies.

\section{Assumptions of Theorem~\ref{thm: convergence rate}}
\label{app:assumptions}
Here we list the necessary technical assumptions for Theorem~\ref{thm: convergence rate}.
\begin{assumption}[Lipschitz continuous gradient]
    \label{as:differentiable} 
    The learning objective $R_\tr(\btheta,w)$ is twice differentiable and has an $L$-Lipschitz continuous gradient for all $w$, i.e.,
    \begin{align*}
    -L I \preceq \nabla_{\btheta}^{2} R_\tr(\btheta,w) \preceq L I,
    \end{align*}
    where $I$ is the identity matrix.
\end{assumption}
This assumption does not require convexity of the learning objective $R_\tr(\btheta,w)$; it only ensures that the gradient does not change arbitrarily fast.

\begin{assumption}[Unbiased stochastic gradients with bounded variance]
    \label{as:noise}
    Let $\tilde{\nabla}_{\btheta}$ be the noisy gradient given by the gradient generator and we assume it satisfies:
    \begin{align*}
    \bE[\tilde{\nabla}_{\btheta} R_\tr(\btheta,w)]=\bE[{\nabla}_{\btheta} R_\tr(\btheta,w)],\
    \bE_{\tilde{\nabla}_{\btheta}}\left[\left\|\tilde{\nabla}_{\btheta} R_\tr(\btheta,w)-\nabla_{\btheta} R_\tr(\btheta,w)\right\|^{2}\right]\leq \sigma^{2}
    \end{align*}
    for some constant $\sigma^{2}$ and all $w$, where $\|\cdot\|$ denotes the $L_2$ norm.
\end{assumption}
This is a standard assumption in stochastic optimization \citep{ghadimi2013stochastic, garrigos2023handbook}.

\begin{assumption}
\label{as:Sensitivity}
We assume the WE subroutine satisfies two stability conditions:
\begin{enumerate}
\item[(W1)] (\emph{Nonexpansive gradient projection}) Let $\Pi_Q$ be the Euclidean projection
onto a closed convex set $Q$, then $$\|\Pi_Q(u)-\Pi_Q(v)\|_1 \le \|u-v\|_1,\; \forall~u,v \in \mathbb{R}^n,$$
where $\|\cdot\|_1$ denotes the $L_1$ norm.
\item[(W2)] (\emph{Learning-rate adjustment}) For each outer iteration $s$ and all inner iterations $t=0,\dots,T$,
\[ \big\|\eta_t \tilde{\nabla}_{\calW} \widehat J(\calW_{s,t};\btheta_s)\big\|_1 \le \left\|\alpha_s \tilde{\nabla}_{\btheta} R_\tr(\btheta_s,\calW_{s})\right\|^2, \]
where $\calW_{s,t}$ denotes the weight vector at inner step $t$ and outer step $s$, $\calW_s=\calW_{s-1,T}$, and $\eta_t$ and $\alpha_s$ are the inner-loop and outer-loop learning rates, respectively.
\end{enumerate}
\end{assumption}
W1 ensures that the projection operation does not increase the magnitude of the weight updates, ensuring numerical stability.
W2 enforces proportional scaling between the inner- and outer-loop updates, ensuring that the cumulative weight updates, measured in $L_1$ norm to capture the total effect across all samples, remain controlled relative to the classifier gradient, measured in squared $L_2$ norm to capture its total energy, while still allowing sufficient flexibility for WE.


\section{Proof of Theorem~\ref{thm: convergence rate}}
\label{sec:appendix_proof}
\begin{proof}
We analyze the ADIW algorithm with inner weight updates \[\calW_{s,t} := (w_{i,s,t}:=w(\bz_{i,s, t}^\tr))_{i=1}^{n_\tr} \in\mathbb{R}^{n_\tr}, \; t=0,\cdots, T,\] and outer classifier updates.
At each outer iteration $s$, the classifier parameters $\theta$ are updated by
\begin{align}
\btheta_{s+1} = \btheta_s-\alpha_s \tnabla R_\tr(\btheta_s,\calW_{s,0}), \; s=0,\cdots,S,
\end{align}
where $\alpha_s>0$ is the outer learning rate.
During the outer iteration $s$, at each inner iteration $t$, the weights are updated by projected gradient descent~(PGD):
\begin{equation}\label{eq:pgd-w-update}
\calW_{s,t+1} \;=\; \Pi_Q\!\Big(\calW_{s, t} - \eta_{t}\,\tilde{\nabla}_{\calW}\widehat J(\calW_{s, t};\btheta_s)\Big),
\end{equation}
where $\eta_{t}>0$ is the inner learning rate and $\Pi_Q$ denotes projection onto $Q$.

\paragraph{Bounding weight changes.} Fix an outer iteration $s$.
Given that the inner PGD runs $T$ steps,
\begin{align*}
    \calW_{s,0}\ \xrightarrow{\ \text{PGD step in Eq.\ } \eqref{eq:pgd-w-update}\ }\ \calW_{s,1}\ \xrightarrow{\ }\ \cdots\ \xrightarrow{\ }\ \calW_{s,T},
\end{align*}
we have the total weight change across the inner loop
\begin{align}
    \calW_{s,T}-\calW_{s,0}
= \sum_{t=0}^{T-1}\big(\calW_{s,t+1}-\calW_{s,t}\big).
\end{align}
Using the triangle inequality, we have
\begin{align}
    \|\calW_{s,T}-\calW_{s,0}\|_1
\le \sum_{t=0}^{T-1}\|\calW_{s,t+1}-\calW_{s,t}\|_1.
\end{align}
For each inner iteration,
\begin{align}
\|\calW_{s,t+1}-\calW_{s,t}\|_1
&= \big \| \Pi_Q(\calW_{s,t}-\eta_t\tilde{\nabla}_{\calW}\widehat J(\calW_{s,t};\btheta_s))
      -\Pi_Q(\calW_{s,t})\big\|_1 \notag\\
      &\le \big \|(\calW_{s,t}-\eta_t\tilde{\nabla}_{\calW}\widehat J(\calW_{s,t};\btheta_s))-\calW_{s,t}\big\|_1 \notag\\
      &= \,\big\|\eta_t \tilde{\nabla}_{\calW}\widehat J(\calW_{s,t};\btheta_s)\big\|_1 \notag\\
      &\le \alpha_s^{2} \|\tnabla R_\tr(\btheta_s,\calW_{s})\|^2,
\end{align}
where the first equation is due to Eq.~\eqref{eq:pgd-w-update} and the fact that at the current iteration, $\calW_{s,t}$ lies in $Q$ and thus $\calW_{s,t}=\Pi_Q(\calW_{s,t})$; the second inequality is due to the nonexpansiveness of $\Pi_Q$ (Assumption W1); and the last inequality is due to the stability of weight gradient (Assumption W2).
Then summing over \(t=0,\dots,T-1\) yields
\begin{align}
\label{eq:inner-loop-weight-bound}
\|\calW_{s,T}-\calW_{s,0}\|_1
\le T \alpha_s^{2} \,\|\tgrad\|^2,
\end{align}
where $\tgrad$ denotes $\tnabla R_\tr(\btheta_s,\calW_{s})$
Denote $\calW_{s} := \calW_{s-1,T}$, the total weight change across the outer loop satisfies
\begin{align}
\calW_{s+1}-\calW_s
&= \calW_{s,T}-\calW_{s-1,T}\notag\\
&=(\calW_{s, T}-\calW_{s,0})+(\calW_{s,0}-\calW_{s-1,T}).\label{eq:weight-change-across-outer-loop}
\end{align}
We note that the second term of Eq.~\eqref{eq:weight-change-across-outer-loop} is $0$ as $\calW_{s,0}$ is the starting point for the inner loop at outer iteration $s$, which is inherited from the last step $\calW_{s-1,T}$.
Combining the above results leads to
\begin{align}
    \|\calW_{s+1}-\calW_s\|_1 
    &\le \|\calW_{s,T}-\calW_{s,0}\|_1 + \|\calW_{s,0}-\calW_{s-1,T}\|_1 \notag\\
    &\le T \alpha_s^{2} \,\|\tgrad\|^2.\label{eq:weight-change}
\end{align}

\paragraph{Bounding risk decrease.}
By \emph{Taylor's theorem}, there exists $\btheta'$ on the line segment between $\btheta_{s}$ and $\btheta_{s+1}$ such that
\begin{align}
R_\tr(\btheta_{s+1},\calW_{s})&=R_\tr\left(\btheta_s-\alpha_s \tgrad,\calW_{s}\right)\notag\\
&=R_\tr(\btheta_s,\calW_{s})-\alpha_s \tgrad^{T}\nabla_{\theta} R_\tr(\btheta_s,\calW_{s})+\frac{\alpha_s^{2}}{2}\tgrad^{T}\nabla_{\theta}^{2} R_\tr(\btheta',\calW_{s})\tgrad\notag\\
&\leq R_\tr(\btheta_s,\calW_{s})-\alpha_s \tgrad^{T}\nabla_{\theta} R_\tr(\btheta_s,\calW_{s})+\frac{\alpha_s^{2}L}{2}\left\|\tgrad\right\|^{2},\label{eq:risk-decrease}
\end{align}
where the last inequality is due to the Lipschitz continuous gradient (Assumption~\ref{as:differentiable}).
Changing the weights from $\calW_s$ to $\calW_{s+1}$ gives
\begin{align}
R_\tr(\btheta_{s+1},\calW_{s+1})-R_\tr(\btheta_{s+1},\calW_{s})
&=\frac{1}{n_\tr}\sum_{i=1}^{n_\tr}w_{i,s+1}\ell(\boldsymbol{f}_{\btheta_{s+1}}(\bx_i),y_i)-\frac{1}{n_\tr}\sum_{i=1}^{n_\tr}w_{i,s}\ell(\boldsymbol{f}_{\btheta_{s+1}}(\bx_i),y_i)\notag\\
&=\frac{1}{n_\tr}\sum_{i=1}^{n_\tr}\left[(w_{i,s+1}-w_{i,s})\ell(\boldsymbol{f}_\mathrm{\theta_{s+1}}(\bx_i),y_i)\right]\notag\\
&\leq \frac{M}{n_\tr}\|\calW_{s+1}-\calW_{s}\|_1\notag\\
&\leq \frac{M T \alpha_s^{2}}{n_\tr} \,\|\tgrad\|^2,\label{eq:risk-decrease-change-weights}
\end{align}
where the first inequality applies the upper bound of the loss, and the last inequality is due to Eq.~\eqref{eq:weight-change}.
Combining Eq.~\eqref{eq:risk-decrease} and~\eqref{eq:risk-decrease-change-weights} gives
\begin{align}
R_\tr(\btheta_{s+1},\calW_{s+1})
\leq R_\tr(\btheta_s,\calW_{s})-\alpha_s \tgrad^{T}\grad+\left(\frac{\alpha_s^{2}L}{2}+\frac{M T \alpha_s^{2}}{n_\tr}  \right)\left\|\tgrad\right\|^{2},
\end{align}
where $\grad$ denotes $\nabla_{\theta} R_\tr(\btheta_s,\calW_{s})$. 

\paragraph{Convergence via telescoping sum.}
Taking the expected value gives us
\begin{align}
\bE[R_\tr(\btheta_{s+1},\calW_{s+1})|\btheta_{s}]
&\leq R_\tr(\btheta_s,\calW_{s})-\alpha_s\bE\left[\tgrad^{T}\grad\big|\btheta_{s}\right]+\left(\frac{\alpha_s^{2}L}{2}+\frac{M T \alpha_s^{2}}{n_\tr}  \right)\bE\left[\left\|\tgrad\right\|^{2}\Big|\btheta_{s}\right]\notag\\
&= R_\tr(\btheta_s,\calW_{s})-\alpha_s\left\|\grad\right\|^{2}+\left(\frac{\alpha_s^{2}L}{2}+\frac{M T \alpha_s^{2}}{n_\tr}  \right)\bE\left[\left\|\tgrad\right\|^{2}\Big|\btheta_{s}\right],
\end{align}
where the equality is due to the unbiased stochastic gradients (Assumption~\ref{as:noise}).
Since
\begin{align}
\bE\left[\left\|\tgrad\right\|^{2}\Big|\btheta_{s}\right]
&=\bE\left[\left\|\tgrad-\grad+\grad\right\|^{2}\Big|\btheta_{s}\right]\notag\\
&=\bE\left[\left\|\tgrad-\grad\right\|^{2}+\|\grad\|^{2}+2\left(\tgrad-\grad\right)^{T}\grad\Big|\btheta_{s}\right]\notag\\
\label{eq:sigma}%
&\leq \sigma^{2}+\|\grad\|^{2},
\end{align}
where the last inequality comes from the unbiased stochastic gradients and bounded variance (Assumption~\ref{as:noise}), then we have
\begin{align}
\bE[R_\tr(\btheta_{s+1},\calW_{s+1})|\btheta_{s}]
&\leq R_\tr(\btheta_s,\calW_{s})-\left(\alpha_s-\frac{\alpha_s^{2}L}{2}-\frac{M T \alpha_s^{2}}{n_\tr}  \right)\left\|\grad\right\|^{2}+\left(\frac{\alpha_s^{2}L}{2}+\frac{M T \alpha_s^{2}}{n_\tr}  \right)\sigma^{2}.
\end{align}
If we set $\alpha_s$ small enough such that $\alpha_s<\frac{n_\tr}{2MT+nL}$ for all $s$, this simplifies to
\begin{align}
\bE[R_\tr(\btheta_{s+1},\calW_{s+1})|\btheta_{s}]
&\leq R_\tr(\btheta_s,\calW_{s})-\frac{\alpha_s}{2}\|\grad\|^{2}+\left(\frac{\alpha_s^{2}L}{2}+\frac{M T \alpha_s^{2}}{n_\tr}  \right)\sigma^{2}.
\end{align}
Now taking the full expectation,
\begin{align}
\bE[R_\tr(\btheta_{s+1},\calW_{s+1})]
&\leq \bE[R_\tr(\btheta_s,\calW_{s})]-\frac{\alpha_s}{2}\bE\left[\left\|\grad\right\|^{2}\right]+\left(\frac{\alpha_s^{2}L}{2}+\frac{M T \alpha_s^{2}}{n_\tr}  \right)\sigma^{2},
\end{align}
and then rearranging the terms,
\begin{align}
\label{eq:sum}%
\frac{1}{2}\bE\left[\left\|\grad\right\|^{2}\right]
\leq \frac{\bE[R_\tr(\btheta_s,\calW_{s})]-\bE[R_\tr(\btheta_{s+1},\calW_{s+1})]}{\alpha_s}+\left(\frac{\alpha_s L}{2}+\frac{M T \alpha_s}{n_\tr}  \right)\sigma^{2}.
\end{align}
Next summing up Eq.~\eqref{eq:sum} from $s=1$ to $S$,
\begin{align}
\frac{1}{2}\sum_{s=1}^{S}\bE\left[\left\|\grad\right\|^{2}\right]
&\leq\sum_{s=1}^{S}\frac{\bE[R_\tr(\btheta_s,\calW_{s})]}{\alpha_s}-\sum_{s=1}^{S}\frac{\bE[R_\tr(\btheta_{s+1},\calW_{s+1})]}{\alpha_s} + \left(\frac{L}{2}+\frac{M T}{n_\tr}  \right)\sigma^{2}\sum_{s=1}^{S}\alpha_s\notag\\
&=\sum_{s=1}^{S}\frac{\bE[R_\tr(\btheta_s,\calW_{s})]}{\alpha_s}-\sum_{s=2}^{S+1}\frac{\bE[R_\tr(\btheta_s,\calW_{s})]}{\alpha_{s-1}}+\left(\frac{L}{2}+\frac{M T}{n_\tr}  \right)\sigma^{2}\sum_{s=1}^{S}\alpha_s\notag\\
&=\sum_{s=2}^{T}\left(\frac{1}{\alpha_{s}}-\frac{1}{\alpha_{s-1}}\right)\bE[R_\tr(\btheta_s,\calW_{s})]+\frac{\bE[R_\tr(\btheta_1,\calW_{1})]}{\alpha_1}\notag\\
&\quad-\frac{\bE[R_\tr(\btheta_{S+1},\calW_{S+1})]}{\alpha_{S}}+\left(\frac{L}{2}+\frac{M T}{n_\tr}  \right)\sigma^{2}\sum_{s=1}^{S}\alpha_s\notag\\
&\leq \sum_{s=2}^{S}\left(\frac{1}{\alpha_{s}}-\frac{1}{\alpha_{s-1}}\right)BM+\frac{BM}{\alpha_1}+\left(\frac{L}{2}+\frac{M T}{n_\tr}  \right)\sigma^{2}\sum_{s=1}^{S}\alpha_s\notag\\
&=\frac{BM}{\alpha_{S}}+\left(\frac{L}{2}+\frac{M T}{n_\tr}  \right)\sigma^{2}\sum_{s=1}^{S}\alpha_s,
\end{align}
where the last inequality applies the upper bound of the weights and loss.
By setting $\alpha_s=\frac{c}{\sqrt{s}}$ (with $\frac{1}{\alpha_{0}} \triangleq 0$), we obtain
\begin{align}
\frac{1}{2}\sum_{s=1}^{S}\bE\left[\left\|\grad\right\|^{2}\right]
&\leq \frac{BM\sqrt{S}}{c}+\left(\frac{L}{2}\sigma^{2}+\frac{M T \sigma^{2}}{n_\tr}\right)2c\sqrt{S}\notag\\
\label{eq:goal}%
&= \left(\frac{BM}{c}+cL\sigma^{2}+\frac{2cMT\sigma^{2}}{n_\tr}\right)\sqrt{S},
\end{align}
where the inequality follows since $\sum_{s=1}^{S} \frac{1}{\sqrt{S}} \leq 2 \sqrt{S}$.
Therefore, let $m_S=\btheta_{s}$ with probability $\frac{1}{S}$ for all $s\in\{1, \ldots, S\}$, and the expected value of gradient at this point is
\begin{align}
\bE[\|\nabla R_\tr(m_S,\calW_{s})\|^{2}]
&=\sum_{s=1}^{S}\bE\left[\left\|\grad\right\|^{2}\right]\cdot \mathbf{P}\left(m_{S}=\btheta_{s}\right)\notag\\
&=\frac{1}{S}\sum_{s=1}^{S}\bE\left[\left\|\grad\right\|^{2}\right].
\end{align}
Substituting Eq.~\eqref{eq:goal} into this gives us
\begin{align}
\bE[\|\nabla R_\tr(m_S,\calW_{s})\|^{2}]
&\leq \frac{2}{\sqrt{S}}\left(\frac{BM}{c}+cL\sigma^{2}+\frac{2cMT\sigma^{2}}{n_\tr}\right),
\end{align}
which concludes the proof.
\end{proof}

\section{ADIW with Hidden-Layer-Output Transformation}
\label{sec:appendix-algo}
The algorithm of ADIW with the hidden-layer-output transformation is presented in Algorithm~\ref{alg:ADIW-F}.
Compared with Algorithm~\ref{alg:ADIW}, it involves substantially more steps, as it requires partitioning the training and validation data by class within each mini-batch and performing weight estimation in a class-wise manner. 
This results in fewer samples per estimation, potentially causing unstable weight estimation and requiring additional handling for cases with insufficient samples.
Furthermore, a class-prior ratio is needed to be estimated before training and applied to the computed weights.
Therefore, we recommend using the loss-value transformation as the default configuration in ADIW, as in Algorithm~\ref{alg:ADIW}.

\begin{algorithm}
    \caption{Accelerated dynamic importance weighting (with hidden-layer-output transformation, in a mini-batch).}
    \label{alg:ADIW-F}
    \begin{algorithmic}
        \REQUIRE a training and validation minibatch $\Str$ and $\Sval$ with index set          $\mathcal{I}_{\Str}$ and $\mathcal{I}_{\Sval}$,
         current model $\boldsymbol{f}_{\theta}$,
         weight vector $\calW$,
         WE objective $\widehat J$ with constraint set $Q$,  
         step size $\eta_t$, 
         number of WE iterations $T$, class-prior ratio $\{r_y^*\}_{y=1}^C$
    \end{algorithmic}
    \begin{minipage}[t]{\textwidth}\vspace{-8pt}
        \begin{algorithmic}[1]
            \STATE \texttt{forward} the input parts of $\Str$ \& $\Sval$
            \STATE \texttt{retrieve} the hidden-layer outputs as $\calZ^\tr$ \& $\calZ^\val$
            \STATE \texttt{apply} L2 normalization to $\calZ^\tr$ \& $\calZ^\val$: $\calZ^\tr \gets \text{L2}(\calZ^\tr)$, $\calZ^\val \gets \text{L2}(\calZ^\val)$
            \STATE \texttt{partition} $\calZ^\tr$ \& $\calZ^\val$ into $\{\calZ^\tr_y\}_{y=1}^C$ \& $\{\calZ^\val_y\}_{y=1}^C$
            \STATE \texttt{partition} $\mathcal{I}_{\Str}$ \& $\mathcal{I}_{\Sval}$ into $\{\mathcal{I}_{\Str, y}\}_{y=1}^C$ \& $\{\mathcal{I}_{\Sval, y}\}_{y=1}^C$
            \FOR{y = 1 to C}
                \STATE \texttt{retrieve} initial weights for class $y$: $\calW_1^y \gets \calW[\mathcal{I}_{\Str, y}]$
                \IF{$|\calZ^\tr_y| < 2$ \textbf{or} $|\calZ^\val_y| < 2$} 
                    \STATE \texttt{skip} class $y$; retain original weights \texttt{(insufficient samples)}
                    \STATE \textbf{continue}
                \ENDIF
            \FOR{t = 1 to T}
                \STATE \texttt{compute} gradients $\nabla_{\calW_t^y} \widehat J$ from $\calZ^\tr_y$ \& $\calZ^\val_y$
                \STATE \texttt{update} $\calW_{t+1}^y \gets \calW_t^y - \eta_t \nabla_{\calW_t^y} \widehat J$
                \STATE \texttt{project} $\calW_{t+1}^y$ onto $Q$
                \ENDFOR
            \STATE \texttt{multiply} all $w_i\in\calW_{T}^y$ by $r_y^*$
            \ENDFOR
            \IF{\texttt{all classes skipped}}
                \STATE \texttt{set} $\mathcal{W}_T \gets \mathbf{0}$ for $\Str$ \texttt{(fallback)}
            \ELSE
                \STATE \texttt{concatenate} all class-wise indices $\{\mathcal{I}_{\Str, y}\}_{y=1}^C$ and weights $\{\calW_T^y\}_{y=1}^C$
                \STATE \texttt{update} weights in $\calW$: $\calW[\mathcal{I}_{\Str}^{\mathrm{cat}}] \gets \calW_T^{\mathrm{cat}}$
                \STATE \texttt{reorder} $\calW_T^{\mathrm{cat}}$ to match $\mathcal{I}_{\Str}$, yielding $\mathcal{W}_T$
            \ENDIF
            \STATE \texttt{weight} the empirical risk $\widehat{R}(\mf)$ by $\mathcal{W}_T$ 
            \STATE \texttt{backward} $\widehat{R}(\mf)$ and \texttt{update} $\theta$
        \end{algorithmic}
    \end{minipage}
\end{algorithm}

\section{Supplementary Experimental Setups}
\label{sec:appendix_setup}
This section introduces additional experimental details, including descriptions of the datasets and base models used in the main paper, and the specific hyperparameter settings for the class-prior shift, label-noise, and ablation study experiments.

Since computational efficiency is a key aspect of our contribution,  
we evaluated the methods not only in terms of predictive performance but also efficiency.
Specifically, {we assessed the computational efficiency at two levels.}
First, we measured the average end-to-end training time per epoch using \texttt{time.perf\_counter()} from Python’s \texttt{time} module, 
with \texttt{torch.cuda.synchronize()} invoked beforehand to ensure all GPU operations had completed.
Second, we provided a fine-grained breakdown of CPU and GPU runtimes across key algorithmic stages,
presented as an ablation study in Section~\ref{sec:ab_profiling} and in Appendix~\ref{app:ab}.
Experiments on ImageNet-1K were run on NVIDIA A100 GPUs, while others were run on NVIDIA Tesla V100 GPUs.
All experiments were implemented using PyTorch 2.5.0. 

\subsection{Datasets and Base Models}
\textbf{Datasets.} 
The detailed information of the datasets used in the paper is listed below:
\begin{itemize}
    \item \emph{Fashion-MNIST} \citep{xiao2017} is a benchmark dataset for the 10-class image classification task. It consists of 70,000 grayscale images of fashion products across 10 categories, each with a resolution of 28*28 pixels. The dataset is divided into 60,000 training and 10,000 test samples. See \url{https://github.com/zalandoresearch/fashion-mnist} for details.
    \item \emph{CIFAR-10} and \emph{CIFAR-100} \citep{krizhevsky2009learning} are benchmark datasets for image classification of real-world objects. Each dataset consists of 60,000 color images of size 32*32 pixels, split into 50,000 for training and 10,000 for testing. CIFAR-10 contains 10 object categories, while CIFAR-100 has 100 fine-grained classes. See \url{https://www.cs.toronto.edu/~kriz/cifar.html} for details. 
    \item \emph{ImageNet-1K}, short for the \emph{ImageNet Large
    Scale Visual Recognition Challenge 2012 (ILSVRC2012)} \citep{russakovsky2015imagenet}, is a large-scale image dataset of real-world objects grouped into 1000 categories. It contains 1,281,167 training data, 50,000 validation data, and 100,000 test data. Each image is a high-resolution color photograph of real-world scenes, providing rich visual diversity and fine-grained semantic coverage. See \url{https://www.image-net.org} for details. 
\end{itemize}

\textbf{Base Models.}
In this work, we adopted LeNet-5 \cite{lecun1998gradient} as the base model for Fashion-MNIST, whose architecture is as follows:
\begin{itemize}[leftmargin=10em]
    \item[0th (input) layer:] (32*32)-
    \item[1st to 2nd layer:] C(5*5,6)-S(2*2)-
    \item[3rd to 4th layer:] C(5*5,16)-S(2*2)-
    \item[5th layer:] FC(120)-
    \item[6th layer:] FC(84)-10,
\end{itemize}
where C(5*5,6) denotes a convolutional layer with 6 filters of size 5*5 followed by a ReLU activation, S(2*2) represents a max-pooling layer with a 2*2 filter and stride 2, and FC(120) indicates a fully connected layer with 120 output units.

The model used for ImageNet-1K was ResNet-18 \citep{he2016deep}, a residual convolutional neural network consisting of 18 layers organized into four residual stages:
\begin{itemize}[leftmargin=10em]
    \item[0th (input) layer:] (224*224*3)-
    \item[1st to 5th layers:] C(7*7, 64)-S(3*3)-[C(3*3, 64), C(3*3, 64)]*2-
    \item[6th to 9th layers:] [C(3*3, 128), C(3*3, 128)]*2-
    \item[10th to 13th layers:] [C(3*3, 256), C(3*3, 256)]*2-
    \item[14th to 17th layers:] [C(3*3, 512), C(3*3, 512)]*2-
    \item[18th layer:] Global Average Pooling-FC(1000),
\end{itemize}
where C(7*7, 64) denotes a convolution with 64 filters of size 7*7 followed by batch normalization and ReLU activation, S(3*3) represents a max-pooling layer with a 3*3 filter and stride 2, [·, ·]*2 indicates two consecutive residual blocks, each consisting of two 3*3 convolutions followed by batch normalization, with ReLU applied after the first convolution and after the residual addition, and FC(1000) denotes a fully connected layer with 1000 output units.

The model used for CIFAR-10 and CIFAR-100 was a CIFAR-style ResNet-18 variant. Unlike the ImageNet-1K version, it replaces the initial 7*7 convolution and max-pooling layers with a single 3*3 convolution of stride 1.
The remaining structure follows the same four residual stages and concludes with a global average pooling layer followed by a fully connected layer producing 10 or 100 outputs for classification.

\subsection{Class-prior-shift Experiments}
\label{app:cp-exp}
Here we present the hyperparameters used in the class-prior-shift experiments.
For class-prior-shift experiments, the initial learning rate was 0.1 and decayed by a factor of 0.1 every 100 epochs, for a total of 400 epochs. 
The learning rates for weight estimation were 0.0001 for ADIW-KM, 0.005 for ADIW-KL, 0.01 for ADIW-LS, and 0.0005 for ADIW-WA.
The kernel width was 0.1 for DIW and ADIW-KM, and was 5 for ADIW-KL, and 1 for ADIW-LS.

Across all experiments, including the class-prior-shift and the label-noise experiments, the batch size was fixed at 256, and the number of weight estimation iterations was set to 1. For ADIW-LS, the regularization parameter $\lambda$ was set to 1e-05.
For ADIW-WA, the gradient penalty coefficient $\kappa$ was set to 10. The critic $\Phi_\nu$ was trained using the Adam optimizer with a learning rate of 1e-04. The first 50 minibatches were used for pre-training, and each minibatch was used for 3 critic updates.

\subsection{Label-noise Experiments}
\label{app:ln-exp}
In the Fashion-MNIST experiments, the initial learning rate was set to 0.1 and decayed by a factor of 0.1 every 100 epochs, for a total of 400 epochs. The weight decay was fixed at 1e-07. 
The learning rates for weight estimation were 0.1 for ADIW-KM, 0.5 for ADIW-KL, 0.01 for ADIW-LS, and 0.01 for ADIW-WA.

In the CIFAR-10/100 experiments, the initial learning rate was 0.0005, decaying by a factor of 0.1 every 100 epochs, for a total of 400 epochs.
The weight decay was set to 0.0001 for CIFAR-10 and 0.0005 for CIFAR-100. For CIFAR-10, the weight estimation learning rates were 0.1 for ADIW-KM, ADIW-KL, and ADIW-LS, and 0.01 for ADIW-WA. For CIFAR-100, the rates were 0.5 for ADIW-KM, 0.1 for ADIW-KL, 0.1 for ADIW-LS, and 0.005 for ADIW-WA.

In the ImageNet-1K experiments, the initial learning rate was set to 0.0002 and decayed by a factor of 0.1 every 100 epochs, for a total of 300 epochs. The weight decay was 1e-05. The learning rates for weight estimation were 0.5 for ADIW-KM, ADIW-KL, and ADIW-LS, and 0.003 for ADIW-WA.

For all experiments under the label noise, the kernel width was 0.01 for DIW, ADIW-KM, and ADIW-LS, while for ADIW-KL it was set to 0.1 on CIFAR-100 and 0.5 on the other datasets. 
Other hyperparameters shared across both the class-prior shift and label-noise experiments are provided in Appendix~\ref{app:cp-exp}.

\subsection{Ablation Study}
\label{app:ab}
First we present the configuration of the PyTorch Profiler~\footnote{See \url{https://pytorch.org/docs/stable/profiler.html} for PyTorch Profiler.}. The profiler was configured with a window size of 10 iterations and activated during the selected epochs. 
In each activated epoch, the profiler skipped the first 10 iterations as a warm-up, used the following 2 iterations for stabilization, and then recorded results over the subsequent 10 iterations.
The reported results were averaged across all profiling windows.

Then, we provide the hyperparameter configurations used in the CIFAR-10 experiments under 0.4 symmetric label noise for ADIW with the hidden-layer-output transformation, i.e., the ADIW-F variants.
The kernel width was set to 0.5 for ADIW-KM-F, and 0.3 for both ADIW-KL-F and ADIW-LS-F.
The regularization coefficient $\lambda$ for ADIW-LS-F was 0.001.
All other setups were identical to those used in the ADIW-L experiments (see the configurations for CIFAR-10 in Appendix~\ref{app:ln-exp}).

\section{Supplementary Experimental Results}
\label{sec:appendix_results}
In this section, we present additional experimental results, including a profiling analysis on CIFAR-10, a summary of the classification accuracy of the compared methods, statistical plots of the importance weight distributions, and a summary of the end-to-end training time.

\subsection{Key-stage Profiling on CIFAR-10}
\label{sec:profiling_c10}

\begin{table}
\caption{Stage-wise breakdown of total CPU and CUDA time on CIFAR-10 (seconds; percentages in parentheses).}
\label{tab:profiling_c10}
\resizebox{\textwidth}{!}{%
\begin{tabular}{l|ll|ll|ll|ll}
    \toprule
    \multirow{2}{*}{Stage} & \multicolumn{2}{c|}{DIW} & \multicolumn{2}{c|}{DIW\_Val} & \multicolumn{2}{c|}{ADIW\_CPU} & \multicolumn{2}{c}{ADIW\_GPU}\\
    \cmidrule(lr){2-3} \cmidrule(lr){4-5} \cmidrule(lr){6-7} \cmidrule(lr){8-9}
    & CPU (\%) & CUDA (\%) & CPU (\%) & CUDA (\%) & CPU (\%) & CUDA (\%) & CPU (\%) & CUDA (\%) \\
    \midrule
    Fetch tr data & 1.37 (21.78) & 0.01 (0.23) & 1.44 (21.78) & 0.01 (0.22) & 1.14 (19.16) & 0.01 (0.24) & 1.60 (47.31) & 0.01 (1.04) \\
    Fetch val data &  1.26 (20.07) & 0.01 (0.23) & 1.36 (20.58) & 0.02 (0.39) & 1.08 (18.07) & 0.01 (0.23) & 1.46 (43.34) & 0.01 (1.00) \\
    Forward tr data &  0.05	(0.75) & 0.34 (7.88) & 0.04 (0.68) & 0.33 (7.59) & 0.04 (0.67) & 0.34 (7.61) & 0.05 (1.45) & 0.34 (31.83) \\
    Forward val data &  0.05 (0.87) & 0.34 (8.02) & 0.05 (0.82) & 0.33 (7.56) & 0.04 (0.67) & 0.33 (7.48) & 0.06 (1.65) & 0.34 (31.75) \\
    Get tr loss & 0.00 (0.06) & 0.00 (0.01) & 0.01 (0.08) & 0.00 (0.01) & 0.00 (0.06) & 0.00 (0.01) & 0.00 (0.11) & 0.00 (0.02) \\
    Get val loss & 0.00 (0.06) & 0.00 (0.01) & 0.00 (0.06) & 0.00 (0.01) & 0.00 (0.06) & 0.00 (0.01) & 0.01 (0.18) & 0.00 (0.04) \\
    \rowcolor{lightgray}
    Estimate weights & 3.33 (53.04) & 3.21 (75.35) & 3.47 (52.54) &	3.36 (76.18) & 3.49 (58.60) & 3.37 (76.34) & 0.02 (0.47) & 0.01 (0.48) \\
    Weight tr loss & 0.06 (0.97) & 0.34 (7.97) & 0.08 (1.14) & 0.34 (7.75) & 0.06 (0.95) & 0.34 (7.79) & 0.07 (2.08) & 0.34 (32.58) \\
    Backward data & 0.13 (2.00) & 0.00 (0.00) & 0.13 (1.93) & 0.00 (0.00) & 0.08 (1.35) & 0.00 (0.00) & 0.09 (2.71) & 0.00 (0.01) \\
    Update model & 0.02 (0.39) & 0.01 (0.31) & 0.03 (0.39) & 0.01 (0.30) & 0.02 (0.39) & 0.01 (0.30) &  0.02 (0.71) &	0.01 (1.25) \\
    \midrule
    Total & 6.28 (100.00) & 4.26 (100.00) & 6.61 (100.00) &	4.41 (100.00) & 5.96 (100.00) & 4.41 (100.00) & 3.38 (100.00) & 1.06 (100.00) \\
    \bottomrule
\end{tabular}}

\vspace{2pt}

\footnotesize
\raggedright
The results are averaged over 5 profiling windows at epochs 20, 105, 205, 305 and 390. Each profiling window consists of 10 iterations measured by PyTorch Profiler. ``tr'' and ``val'' denote training and validation data. The ``Estimate weights'' stage is highlighted as the primary target for efficiency improvement.

\end{table}

Table~\ref{tab:profiling_c10} reports stage-wise profiling results on CIFAR-10 under 0.4 symmetric label noise. Unlike the ImageNet-1K case in the main paper, DIW, DIW\_Val, and ADIW\_CPU exhibited similar runtimes, likely because CIFAR-10 was small enough to fit entirely in memory, making CPU-GPU transfers and solver overhead less significant.
Nevertheless, the proposed method (i.e., ADIW\_GPU) achieved a substantial efficiency gain, reducing CPU time from 6.28 s to 3.38 s (about 46\% reduction) and CUDA time from 4.26 s to 1.06 s (about 75\% reduction) compared with DIW, with weight estimation contributing less than 0.5\% on both. These results indicated that ADIW’s efficiency gains were most significant at large scale but still offered clear advantages on smaller datasets.





\begin{table}
\caption{Mean accuracy (standard deviation) in percentage on Fashion-MNIST, CIFAR-10/-100 over the last ten epochs under label noise (3 trials). 
}
\label{tab:ADIW_acc}
\resizebox{\textwidth}{!}{%
\begin{tabular}{c|c|cccccccccc}
    \toprule
    Data & Noise & Val-Only & Uniform & Random & L2R & MW-Net & DIW & ADIW-KM & ADIW-KL & ADIW-LS & ADIW-WA \\
    \midrule
    \multirow{6}{*}{\rotatebox{90}{\textit{F-MNIST}}} & 0.3 p & \makecell{79.29 \\ (0.86)} & \makecell{65.11 \\ (0.23)} & \makecell{71.15 \\ (0.53)} & \makecell{\textbf{88.74} \\ \textbf{(0.47)}} & \makecell{\textbf{86.56} \\ \textbf{(1.79)}} & \makecell{\textbf{89.99} \\ \textbf{(0.17)}} & \makecell{\textbf{88.73} \\ \textbf{(0.66)}} & \makecell{\textbf{89.57} \\ \textbf{(0.18)}} & \makecell{\textbf{90.00} \\ \textbf{(0.11)}} & \makecell{\textbf{87.21} \\ \textbf{(0.36)}} \\
    & 0.4 s & \makecell{79.29 \\ (0.86)} & \makecell{62.88 \\ (0.61)} & \makecell{73.49 \\ (0.60)} & \makecell{84.85 \\ (0.13)} & \makecell{\textbf{85.69} \\ \textbf{(1.03)}} & \makecell{\textbf{89.64} \\ \textbf{(0.04)}} & \makecell{\textbf{88.94} \\ \textbf{(0.14)}} & \makecell{\textbf{89.19} \\ \textbf{(0.11)}} & \makecell{\textbf{89.39} \\ \textbf{(0.09)}} & \makecell{\textbf{89.06} \\ \textbf{(0.35)}} \\
    & 0.5 s & \makecell{79.29 \\ (0.86)} & \makecell{55.43 \\ (0.58)} & \makecell{68.26 \\ (1.12)} & \makecell{\textbf{83.91} \\ \textbf{(0.66)}} & \makecell{82.87 \\ (0.73)} & \makecell{\textbf{89.05} \\ \textbf{(0.28)}} & \makecell{\textbf{87.66} \\ \textbf{(0.49)}} & \makecell{\textbf{88.80} \\ \textbf{(0.25)}} & \makecell{\textbf{88.55} \\ \textbf{(0.16)}} & \makecell{\textbf{87.60} \\ \textbf{(0.05)}} \\
    \midrule
    \multirow{6}{*}{\rotatebox{90}{\textit{CIFAR-10}}} & 0.3 p & \makecell{48.09 \\ (0.27)} & \makecell{67.84 \\ (0.55)} & \makecell{\textbf{88.06} \\ \textbf{(0.28)}} & \makecell{\textbf{89.45} \\ \textbf{(0.08)}} & \makecell{79.31 \\ (0.76)} & \makecell{88.00 \\ (0.14)} & \makecell{\textbf{89.16} \\ \textbf{(0.15)}} & \makecell{\textbf{88.71} \\ \textbf{(0.27)}} & \makecell{87.86 \\ (0.13)} & \makecell{\textbf{88.20} \\ \textbf{(0.30)}} \\
    & 0.4 s & \makecell{48.09 \\ (0.27)} & \makecell{62.92 \\ (0.34)} & \makecell{82.74 \\ (0.21)} & \makecell{82.39 \\ (0.50)} & \makecell{\textbf{68.93} \\ \textbf{(5.43)}} & \makecell{\textbf{85.38} \\ \textbf{(0.45)}} & \makecell{\textbf{86.70} \\ \textbf{(0.06)}} & \makecell{\textbf{87.05} \\ \textbf{(0.18)}} & \makecell{\textbf{84.38} \\ \textbf{(0.68)}} & \makecell{\textbf{86.27} \\ \textbf{(0.06)}} \\
    & 0.5 s & \makecell{48.09 \\ (0.27)} & \makecell{57.09 \\ (1.30)} & \makecell{79.07 \\ (0.09)} & \makecell{78.61 \\ (0.35)} & \makecell{62.65 \\ (1.55)} & \makecell{81.31 \\ (0.30)} & \makecell{\textbf{80.64} \\ \textbf{(0.67)}} & \makecell{\textbf{83.20} \\ \textbf{(0.20)}} & \makecell{\textbf{79.26} \\ \textbf{(1.03)}} & \makecell{\textbf{83.41} \\ \textbf{(0.12)}} \\
    \midrule
    \multirow{6}{*}{\rotatebox{90}{\textit{CIFAR-100}}} & 0.3 p & \makecell{12.18 \\ (0.28)} & \makecell{50.34 \\ (0.21)} & \makecell{55.75 \\ (0.29)} & \makecell{\textbf{61.42} \\ \textbf{(1.27)}} & \makecell{50.32 \\ (1.09)} & \makecell{61.93 \\ (0.44)} & \makecell{\textbf{63.88} \\ \textbf{(0.63)}} & \makecell{\textbf{66.39} \\ \textbf{(0.15)}} & \makecell{52.54 \\ (0.35)} & \makecell{57.19 \\ (0.29)} \\
    & 0.4 s & \makecell{12.18 \\ (0.28)} & \makecell{38.65 \\ (0.21)} & \makecell{53.63 \\ (0.90)} & \makecell{56.05 \\ (0.57)} & \makecell{38.36 \\ (0.13)} & \makecell{61.39 \\ (0.28)} & \makecell{\textbf{62.26} \\ \textbf{(0.30)}} & \makecell{\textbf{63.42} \\ \textbf{(0.43)}} & \makecell{59.24 \\ (0.59)} & \makecell{\textbf{62.38} \\ \textbf{(0.21)}} \\
    & 0.5 s & \makecell{12.18 \\ (0.28)} & \makecell{31.06 \\ (0.27)} & \makecell{48.50 \\ (0.43)} & \makecell{52.00 \\ (0.89)} & \makecell{30.39 \\ (0.36)} & \makecell{\textbf{57.35} \\ \textbf{(0.37)}} & \makecell{\textbf{56.02} \\ \textbf{(0.62)}} & \makecell{\textbf{59.09} \\ \textbf{(0.40)}} & \makecell{52.29 \\ (0.62)} & \makecell{\textbf{57.13} \\ \textbf{(0.42)}} \\
    \bottomrule
\end{tabular}}
\vspace{2pt}

\footnotesize
\raggedright
Best and comparable methods (paired \textit{t}-test at significance level 1\%) are highlighted in bold. 
``p'' and ``s'' denote pair and symmetric label noise, respectively.

\end{table}

\subsection{Summary of Classification Accuracy}
\label{sec:app_acc}
Table~\ref{tab:ADIW_acc} further summarizes the mean accuracy (standard deviation) on Fashion-MNIST, CIFAR-10, and CIFAR-100 over the last ten epochs under label noise, compared with all baseline methods.
According to paired \textit{t}-tests, most ADIW variants achieve accuracy comparable to or better than DIW. Notable exceptions include ADIW-LS on CIFAR-10 with 0.3 pair label noise and on CIFAR-100, as well as ADIW-WA on CIFAR-100 with 0.3 pair label noise.
Overall, ADIW-KM and ADIW-KL demonstrate consistently strong performance across all settings and are therefore recommended in practice.

\subsection{Importance Weight Distributions}
\label{app:w_dist_full}
Figure~\ref{fig:ADIW_wdist_c10_sota_full} presents the weight distributions learned by DIW and all ADIW variants.
Overall, both DIW and ADIW effectively assign higher weights to intact samples and lower weights to mislabeled ones.
Among ADIW variants, ADIW-KL yields the most stable weighting, with weights concentrated around the median for both intact and mislabeled data, whereas other variants, especially ADIW-KM and ADIW-WA, produce broader distributions.

\begin{figure}[t]
    \centering
    \begin{minipage}{\columnwidth}
    \begin{minipage}[c]{0.195\columnwidth}\centering\small DIW \end{minipage}%
    \begin{minipage}[c]{0.195\columnwidth}\centering\small ADIW-KM \end{minipage} \hspace{0.2em}%
    \begin{minipage}[c]{0.195\columnwidth}\centering\small ADIW-KL \end{minipage}%
    \begin{minipage}[c]{0.195\columnwidth}\centering\small ADIW-LS \end{minipage}%
    \begin{minipage}[c]{0.195\columnwidth}\centering\small ADIW-WA \end{minipage}\\
    \begin{minipage}[c]{\columnwidth}
        \includegraphics[width=0.193\columnwidth]{fig/w_box_c10_04s_diw.pdf}
        \includegraphics[width=0.193\columnwidth]{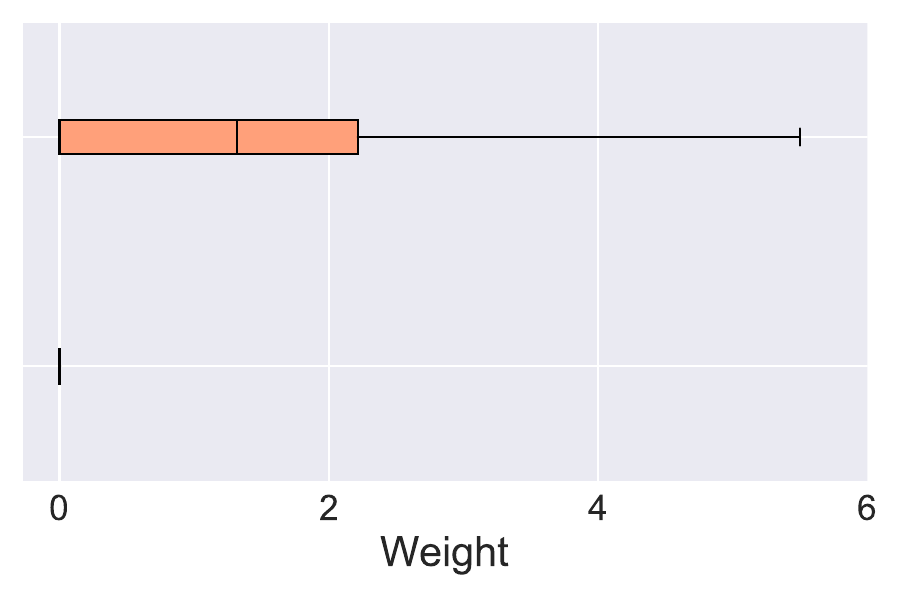}
        \includegraphics[width=0.193\columnwidth]{fig/w_box_c10_04s_kl.pdf}
        \includegraphics[width=0.193\columnwidth]{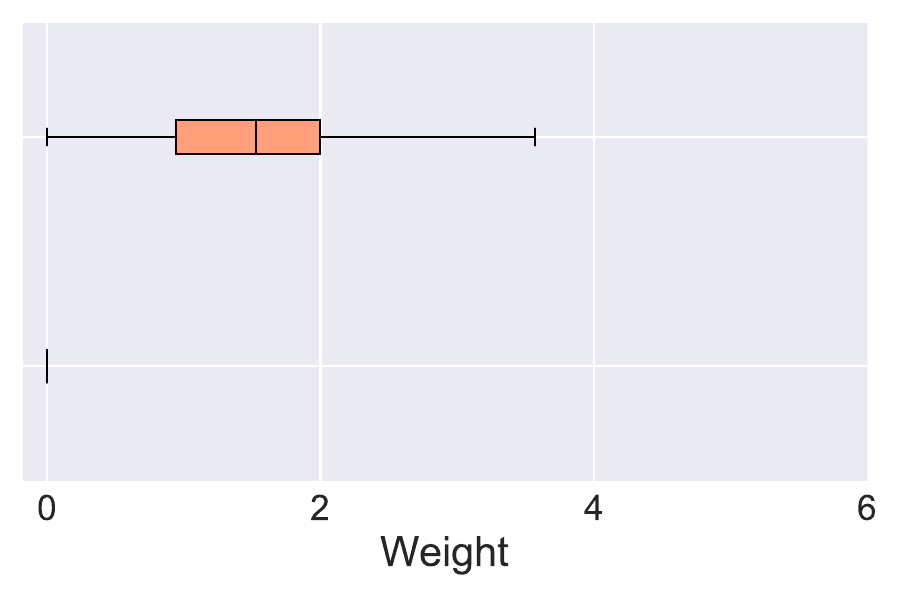}
        \includegraphics[width=0.193\columnwidth]{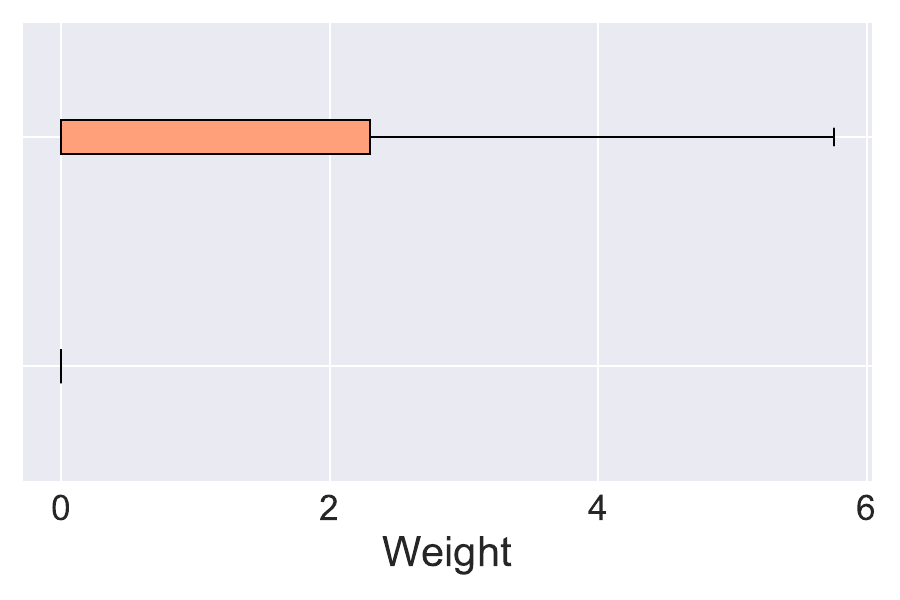}
    \end{minipage}\\
    \begin{minipage}[c]{\columnwidth}
        \includegraphics[width=0.193\columnwidth]{fig/w_dist_c10_04s_diw.pdf}
        \includegraphics[width=0.193\columnwidth]{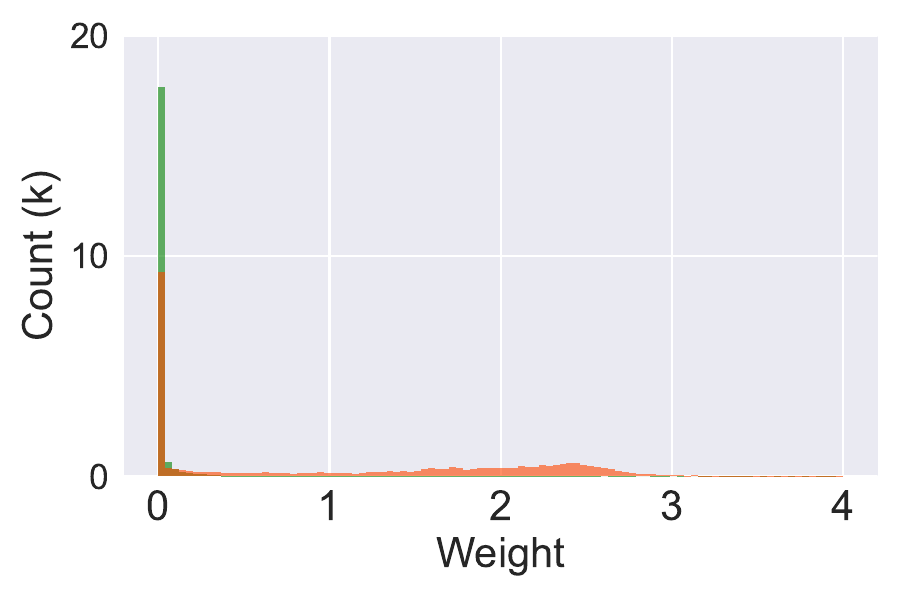}
        \includegraphics[width=0.193\columnwidth]{fig/w_dist_c10_04s_kl.pdf}
        \includegraphics[width=0.193\columnwidth]{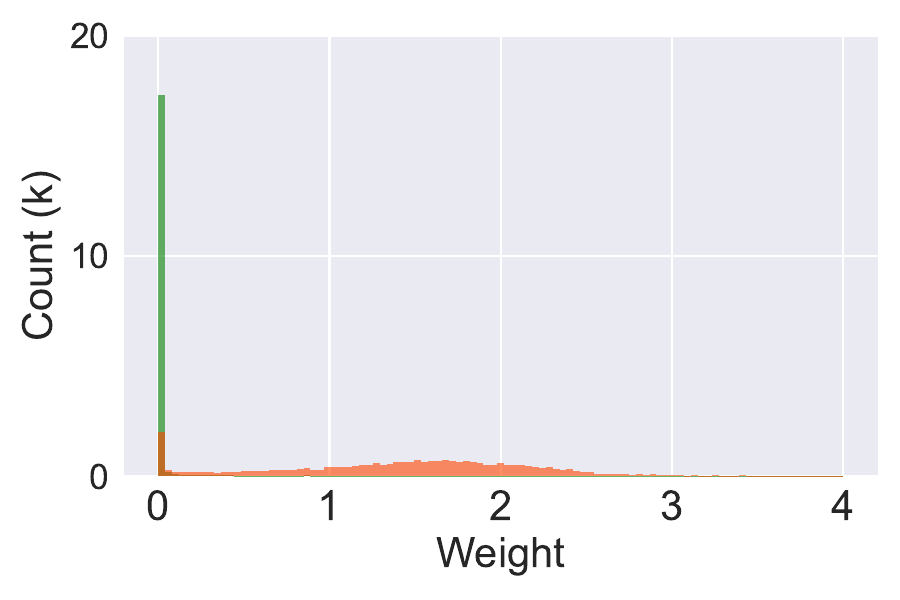}
        \includegraphics[width=0.193\columnwidth]{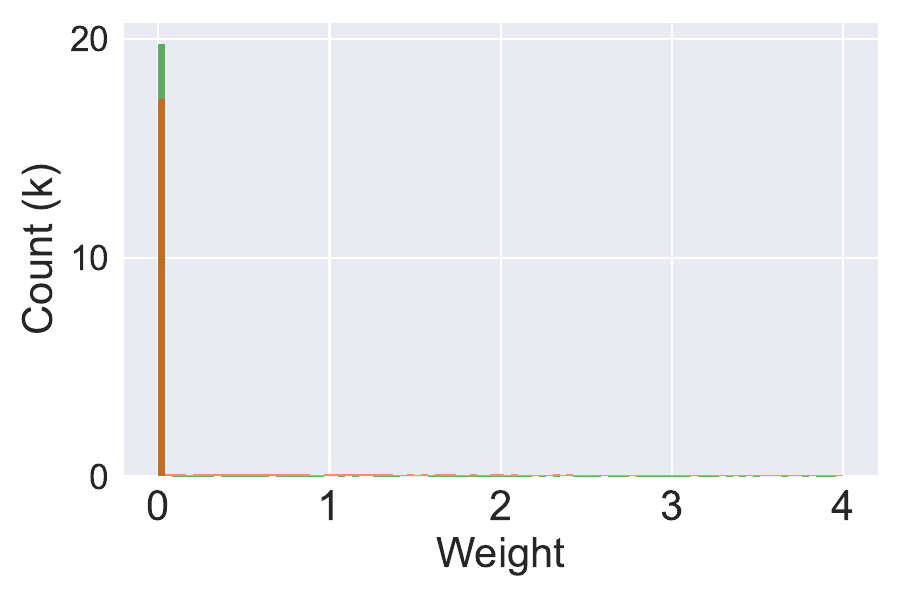}
    \end{minipage}\\
    \caption{Statistics of weight distributions on CIFAR-10 under 0.4 symmetric label noise (top: box plots; bottom: histogram plots).}
    \label{fig:ADIW_wdist_c10_sota_full}
\end{minipage}
\end{figure}

\subsection{Summary of End-to-end Training Time}
\label{app:time}
Table~\ref{tab:efficiency} summarizes the end-to-end per-epoch training time under 0.3 pair, 0.4 symmetric, and 0.5 symmetric label noise settings. 
Figure~\ref{fig:ADIW_time} in the main text presents the 0.4 symmetric case as a representative setting, while this table confirms that similar efficiency trends hold across all label noise configurations.

\begin{table}[t]
\caption{End-to-end training time per epoch$^*$ in seconds on Fashion-MNIST (F-MNIST), CIFAR-10, CIFAR-100, and ImageNet-1K. }
\resizebox{\textwidth}{!}{%
\label{tab:efficiency}
\begin{tabular}{c|c|cccccccc}
    \toprule
    Data & Noise & Uniform & L2R & MW-Net & DIW & ADIW-KM & ADIW-KL & ADIW-LS & ADIW-WA \\
    \midrule
    \multirow{3}{*}{\textit{F-MNIST}} & 0.3 p & 8.13 & 23.85 & 18.87 & 77.25 & 29.63 & 16.07 & 21.09 & 19.48 \\
    & 0.4 s & 8.02 & 23.15 & 18.93 & 80.10 & 29.81 & 21.80 & 24.10 & 19.92 \\
    & 0.5 s & 8.19 & 23.97 & 21.54 & 76.39 & 28.08 & 21.64 & 31.85 & 19.69 \\
    \midrule
    \multirow{3}{*}{\textit{CIFAR-10}} & 0.3 p & 17.85 & 119.66 & 134.40 & 90.40 & 55.04 & 38.48 & 69.59 & 39.59 \\
    & 0.4 s & 17.90 & 121.38 & 152.12 & 109.55 & 49.71 & 38.82 & 67.94 & 39.66 \\
    & 0.5 s & 17.90 & 121.07 & 135.15 & 111.01 & 38.73 & 38.61 & 68.53 & 39.79 \\
    \midrule
    \multirow{3}{*}{\textit{CIFAR-100}} & 0.3 p & 18.49 & 120.80 & 133.67 & 89.60 & 67.24 & 38.74 & 55.60 & 40.00 \\
    & 0.4 s & 18.64 & 120.68 & 153.93 & 110.65 & 67.15 & 38.56 & 50.19 & 39.87 \\
    & 0.5 s & 18.54 & 118.98 & 135.10 & 114.80 & 68.61 & 38.54 & 38.56 & 39.75 \\
    \midrule
    \multirow{1}{*}{\textit{ImageNet-1K}} & 0.4 s & 746.99 & 18917.64 & 20087.89 & 20498.98 & 984.92 & 981.36 & 981.26 & 984.90 \\
    \bottomrule
\end{tabular}}

\vspace{2pt}
\parbox{\columnwidth}{\footnotesize
* Averaged over the full training schedule with the first epoch excluded (300 epochs for ImageNet-1K, 400 epochs for others). L2R, MW-Net, and DIW were run for only four epochs due to prohibitive cost. ``p'' and ``s'' denote pair and symmetric label noise, respectively.
}
\end{table}



\end{document}